\RequirePackage{amsmath}

\documentclass{fundam}

\publyear{22}
\papernumber{2102}
\volume{185}
\issue{1}

\usepackage{xcolor}
\usepackage{tikz}
\usetikzlibrary{arrows,positioning}

\usepackage[utf8]{inputenc}
\usepackage[T1]{fontenc}
\usepackage{latexsym}

\usepackage{amssymb}
\usepackage{makecell}
\usepackage{diagbox}

\usepackage{graphicx}
\usepackage{tabularx}
\usepackage{array}

\usepackage{booktabs}

\usepackage[utf8]{inputenc}
\usepackage{latexsym}
\usepackage[T1]{fontenc}
\usepackage{amssymb}
\usepackage{latexsym}
\usepackage{fix-cm}
\usepackage{array}

\usepackage{url}
\usepackage{hyperref}
\usepackage{stmaryrd}
\usepackage{marvosym}
\usepackage{mdwmath}
\usepackage{mdwtab}
\usepackage{url}
\usepackage{hyperref}
\usepackage{lineno}
\usepackage{graphicx}
\usepackage{tabularx}
\usepackage{array}
\usepackage{mdwmath}
\usepackage{mdwtab}
\usepackage{float}

\usepackage{graphicx}
\usepackage{tikz}
\usetikzlibrary{arrows,positioning}

\usepackage{tabularx}
\usepackage{eqparbox}
\usepackage{booktabs}
\newcommand{\pc}{\mathbf{P}}

\newcommand{\pp}{\mathbb{P}}

\newcommand{\kap}{\kappa}
\newcommand{\ms}{\mathbb{S}}
\title{Granular Generalized Variable Precision Rough Sets and Rational Approximations}
\author{A Mani \corresponding \\
Machine Intelligence Unit\\
Indian Statistical Institute, Kolkata\\
203, B. T. Road, Kolkata (Calcutta)-700108, India\\
Email: \texttt{a.mani.cms@gmail.com, amani.rough@isical.ac.in}\\
Homepage: \url{https://www.logicamani.in}\\
Orcid: \url{https://orcid.org/0000-0002-0880-1035}
\and
Sushmita Mitra\\
Machine Intelligence Unit, Indian Statistical Institute, Kolkata\\
203, B. T. Road, Kolkata-700108, India\\
Email: \texttt{sushmita@isical.ac.in}\\
Orcid: \url{https://orcid.org/0000-0001-9285-1117}}
\begin{document}

\maketitle
\runninghead{A Mani, S Mitra}{Granular VPRS and Rational Approximations}

%\runninghead{A Mani and S. Mitra}{GVPRS and Rational Approximation}

\begin{abstract}
Rational approximations are introduced and studied in granular graded rough sets and generalizations thereof by the first author in recent research papers. The concept of rationality is determined by related ontologies and coherence between granularity, mereology and approximations in the context. In addition, a framework for rational approximations is introduced by her in the mentioned paper(s). Granular approximations constructed as per the procedures of variable precision rough sets (VPRS) are likely to be more rational than those constructed from a classical perspective under certain conditions. This may continue to hold for some generalizations of the former. However, a formal characterization of such conditions is not available in the previously published literature. In this research, theoretical aspects of the problem are critically examined, uniform generalizations of granular VPRS are introduced, new connections with granular graded rough sets are proved, appropriate concepts of substantial parthood are introduced, their extent of compatibility with the framework is accessed, and the framework is extended. Basic assumptions are explained in detail, and additional examples are constructed for readability. Furthermore, meta applications to cluster validation, image segmentation and dynamic sorting are invented. Extensions to direct generalizations of VPRS such as probabilistic rough sets are a natural consequence of the work.
\end{abstract}

\begin{keywords}
Variable Precision Rough Sets, Rationality, Rational Approximations, Substantial Parthood, High Granular Partial Algebras, Contamination Problem, Rough Mereology, Cluster Validation, Graded Rough Sets, Medical Imaging, Generalized RIFs
\end{keywords}

\section{Introduction}

Functionally, an assumption, theory or construct is \emph{rational} just in case the consequences are justified, and they satisfy a specific beneficial or normative pattern. Concepts of rationality in a context are therefore about admissible inferences and additional rules that may be at a higher meta level or semantic domain. An overview of different theories of rationality, and rational inference across multiple fields of study can be found in \cite{knjbsep2020}. A specific version of rationality, called bounded rationality in which possible inferences are bounded or limited has found much application in applied logic. Bounded rationality is additionally relevant in the context of rough set applications. However, the boundedness aspect need not be realizable always in applications. 

Knowledge, and specifically granular knowledge in general rough sets from the perspective of artificial intelligence and human reasoning have always been of interest as can be seen in \cite{zpsk07,zpb,op84,am909,am9411,am3930,am9969,ppm2}. While various facets of rationality like the issue of contamination avoidance \cite{am240} have been investigated, the central focus has not been on rationality of approximations except in a few recent papers \cite{am202236} of the first author. In relation to possible concepts of rational approximations, it should be noted that it is about the ontology of the approximation realized at least partly in their construction that confers rationality. Arguably, variable precision rough sets \cite{zw,kz} provide a certain amount of rationality to rough approximations because approximations require a minimum level of interaction between granules and the thing being approximated. However, the entire context or external factors may be involved in inferring that an approximation is indeed rational or that it follows certain patterns. Complex questions relating to rationality pursued in this research can be seen even in simple examples like the following: 
\begin{example}
People have different reasons for liking pets, no preference for keeping them, or not liking them at all. These attitudes may be seen as the result of their responses to a number of related questions. It is easy to see that those keeping pets must be  likely to respond with pet-loving views on most questions, and play down the downsides (pets may be full of harmful microorganisms). Certain key questions relating to their ability to maintain strict hygiene, provide care, and devote time can determine the rationality of pet keepers. In such a scenario, the concept of a rational pet keeper can be approximated correctly (through VPRS) by structuring the questions explicitly or implicitly. This is about an object being rational, and the implicit or explicit structuring would be about definability. Additionally, it may be possible to badly approximate the concept of a potential pet keeper as that of a non-pet keeper. This is not the only way in which an approximation may not be rational. A mutually inconsistent set of approximations is also a sign of the approximation process lacking rationality. This inconsistency may be deducible from within the system or without.        
\end{example}

It should be noted that concepts of knowledge can be associated from other perspectives (apart from  those based on Pawlakian ideas, and mereological axiomatic granular perspectives mentioned above) such as classical granular computing \cite{ya01,zl1997}, interpretations of modal logic \cite{ppm2,amedit,yl96}, constructive logic \cite{pp2018,jpr}, concept analysis \cite{yy2016,cd9,am9969} and machine learning. However, ideas of knowledge are not well-developed in every generalization of rough sets, and soft computing. 

Reasonable and rational approximations in the context of general rough sets are essential for applications to contexts that require high quality approximations, or predictions (especially when robustness cannot be expected). This is very relevant in applications to human learning, automated evaluation frameworks in education \cite{am5559}, general rough analysis of cluster validation and other areas. Apart from these, a number of potential application areas like intelligent robust image segmentation with unlearning, epidemiology (where feature selection fails badly) and medical diagnostics can be indicated. 

Specifically, it is known in the medical literature \cite{tost2017}, that simply removing tumors (surgically or otherwise) from an organ may lead to new tumors appearing in it with greater intensity. This suggests that the idea of tumors being malignancies of a totally spatially localized kind is flawed; however, it has been the dominant assumption in most medical imaging studies using machine learning (including those using rough sets). It is something dictated by researchers through the methods employed. The proposed generalizations and framework are additionally intended to address the issue from multiple perspectives at a theoretical level.      

The problem of automatic sorting and distributing grains and other materials through machine learning (especially, computer vision \cite{hpra2021}) methods has witnessed a number of disconnected and often opaque contributions. Methods based on granular VPRS that use ideas of rationality are proposed in this research.  

In the application contexts mentioned, a rich collection of mereological relations such as \emph{is an essential part of, is an integral part of, is an apparent part of, is a substantial part of, is a functional part of} and related expression may be found. For example, the process of solving systems of linear equations involves a very large number of parts of the types mentioned. How does one proceed to identify the appropriate ones especially when many nonstandard solution strategies are possible? How does one accommodate alternative conceptions (instead of dismissing them as mistakes) \cite{am5559,sands2018}? Cladistics of various kinds including biology \cite{diana1998} additionally involve more complex parthood relations.  

Variable precision rough sets \cite{zw,kz} and generalizations thereof (see for example \cite{qiao18,am24,yy3w2011,yec2017,zhao09,sdz05}) have been extensively studied from application-oriented perspectives; however, theoretical aspects have not been attended to in a ontological as opposed to a subjective probabilist perspective \cite{xwzl2013,yy3w2011}. Modal connections of some are mentioned in \cite{amk2018}. However, models for a number of variants are not known, and the problem of imposing assumptions remains. If partial possibilities as in \cite{whh22} are to be used, then frameworks for rationality are necessary. The use of classification error or its equivalents (interpreted in numerical terms) and automatic versions based on ideas of risk minimization in decision tables \cite{ywxy15} have their shortcomings. \emph{They certainly do not represent the ontology involved in the context. However what do they actually represent? What do they represent when generalized variants of such functions are used, and which hybrid cases are equivalent to these? Such questions are essentially about the kind of rationality that may be associated with the techniques}.

This research is organized as follows: in the next section some background is provided for convenience, a discussion on rationality and semantic domains is part of the third section. Mereology used in this research is explained and situated in the following sections next. 
Generalizations of rough inclusion functions are reviewed and extended in the sixth section from an axiomatic perspective.  VPRS and granular rough set-theoretical generalizations that may cover some hybrid variants are reviewed or proposed in the following section (concrete concepts of substantial parthood are additionally introduced in a subsection), and new results on the connection of generalized VPRS with graded rough sets are proved in the following section. The granular framework for rational generalized VPRS is proposed in the ninth section. Illustrative examples are constructed next, subsequently the compatibility of granular VPRS with the framework is explored. Meta applications to cluster validation, medical imaging, and dynamic sorting are considered in the penultimate section. Directions for further research are indicated in addition.

\section{Notation and Terminology}

Quantifiers are uniformly enclosed in braces for easier reading. Thus,\\ $\forall a \exists b \, \Phi(a,b) $ is the same as $(\forall a \exists b)\, \Phi (a, b) $.   

Conditional implications of the form \emph{for every $x$ whenever $\Phi(x)$ holds, then $\Psi(x)$ holds as well}, are replaced (whenever possible) by 
\[(\forall x) (\Phi(x) \longrightarrow \Psi(x))\]

The following axiom will be assumed as well:

\textsf{Emergency Axiom: If a theorem is false only because of a set $A$ being nonempty, then it will be tacitly assumed that $A$ is not empty}.

In quasi or partially ordered sets, sets of mutually incomparable  elements are called \emph{antichains}. Some of the basic properties may be found in \cite{gg1998,koh}. The possibility of using antichains of rough objects for a possible semantics was mentioned in \cite{am9501,am3690} and has been developed subsequently in \cite{am6999,am9114}. The semantics invented in the paper is applicable for a large class of operator based rough sets including specific cases of rough Y- systems (\textsf{RYS}) \cite{am240} and other general approaches like \cite{cd3,cc5,yy9,it2}. In \cite{cd3,cc5,gc2018}, negation like operators are assumed in general and these are not definable operations in terms of order related operations in an algebraic sense (additionally, Kleene negation is not a definable operation in the situation).

\subsection{Partial Algebras}

For basics of partial algebras, the reader is referred to \cite{bu,lj}.
\begin{definition}
A \emph{partial algebra} $P$ is a tuple of the form \[\left\langle\underline{P},\,f_{1},\,f_{2},\,\ldots ,\, f_{n}, (r_{1},\,\ldots ,\,r_{n} )\right\rangle\] with $\underline{P}$ being a set, $f_{i}$'s being partial function symbols of arity $r_{i}$. The interpretation of $f_{i}$ on the set $\underline{P}$ should be denoted by $f_{i}^{\underline{P}}$; however, the superscript will be dropped in this paper as the application contexts are simple enough. If predicate symbols enter into the signature, then $P$ is termed a \emph{partial algebraic system}.   
\end{definition}

In this paragraph the terms are not interpreted. For two terms $s,\,t$, $s\,\stackrel{\omega}{=}\,t$ shall mean, if both sides are defined then the two terms are equal (the quantification is implicit). $\stackrel{\omega}{=}$ is the same as the existence equality (sometimes written as $\stackrel{e}{=}$) in the present paper. $s\,\stackrel{\omega ^*}{=}\,t$ shall mean if either side is defined, then the other is and the two sides are equal (the quantification is implicit). Note that the latter equality can be defined in terms of the former as 
\[(s\,\stackrel{\omega}{=}\,s \, \longrightarrow \, s\,\stackrel{\omega}{=} t)\&\,(t\,\stackrel{\omega}{=}\,t \, \longrightarrow \, s\,\stackrel{\omega}{=} t) \]
\subsection{Other Concepts}

\begin{definition}
If $A, B\in\mathcal{S}\subseteq \wp(S)$, with $\mathcal{S}$ being closed under intersection, $S$ being a finite set and if $\# ()$ is the cardinality function, then the equation 
\begin{equation*}\label{rif0}
\nu(A, B) = \left\lbrace  \begin{array}{ll}
 \dfrac{\# (A\cap B)}{\# (A)} & \text{if } A\neq \emptyset\\
 1 & \text{if } A= \emptyset\\
 \end{array} \right. \tag{K0}                                                                                                             
\end{equation*}
defines the \emph{rough inclusion function (or the classical RIF)} $\nu$. 
\end{definition}

The \emph{classification error function} $C$ is defined by $C(A, B) = 1- \nu (A, B)$ for any $A$ and $B$. It is interpreted as the error in classifying $B$ as $A$. VPRS approximations and variants can be defined with either of the two functions or generalizations thereof. 

Clearly, if $A\subset B$, then $\nu(A, B) =1$; however, $\nu(B, A) <1$. Thus it forgives excesses and is already faulty. Moreover, if $E$ is another subset, then $A\subset B$ implies $\nu(B, E)\leq \nu(A, E)$ -- this property is less common.

\section{Rationality and Semantic Domains}

Semantic domains are important in any logical or mathematical approach that focuses on meaning and models. It suffices to specify these without completely formalizing them within specific domains of knowledge. One way of identifying semantic domains in contexts involving approximate reasoning or rough sets is through the type of objects involved in the discourse. This can be specified in a number of ways such as through direct observation of all attributes, observation of parts of an object's attributes, approximations of objects, operations permitted (or admissible predicates) on objects. 

For example, a teacher's understanding of a lesson is very different from that of a student, and each may be approximating a certain knowledge structure (see \cite{bdtmp2008} for a more detailed discussion). In the context, more than six semantic domains may be constructed for the purpose of modeling from different perspectives simply because multiple agents and knowledge hierarchies are involved. 

Consider the sentences 
\begin{description}
 \item [S-A]{The golden mountain is golden and a mountain.}
 \item [S-B]{The tallest mountain is tall and a mountain.}
\end{description}
It is common knowledge that S-A is an unverifiable assertion, while S-B is a verifiable fact. However relative to a knowledgebase in which the concept of \emph{being tall} is missing, both S-A and S-B would be unverifiable. If related semantic domains are formalizable, then they would be different. 

The above examples illustrate the motivation for referring to semantic domains.

In the general rough context, the concepts of definite and rough (or vague) objects are specified through one or more approximation operators and additional constraints. Concepts of crisp objects are additionally specified in the same way or may be designated as such. In the case of Pawlakian (or classical) rough sets, if the lower approximation of an object $X$ coincides with itself, then $X$ is said to be crisp or definite. Objects that are not definite are \emph{rough}. There are several ways of representing rough objects in terms of crisp objects such as a pair of definite objects $(A, B)$ with $A \subset B $. It may additionally be reasonable to think of sets of objects that properly contain common maximal crisp objects as a rough object, or take orthopairs \cite{gcd2018} as the primary objects of interest. For a longer list see \cite{am501}. Naturally, these lead to different rough semantic domains.

Classical rough sets (see \cite{zpb}) starts from approximation spaces (derived from information tables) of the form $\left\langle S,\,R \right\rangle $, with $R$ being an equivalence on the set $S$. The Boolean algebra with operators on the power set $\wp (S)$ with lower and upper
approximation operations forms a model that does not describe the rough objects alone. This extends to arbitrary general approximation spaces where $R$ is permitted to be any relation or even to those based on covers. The \emph{classical semantic domain} associated with such classes of models may be understood in terms of the collection of restrictions on possible
objects, predicates, constants, functions and low level operations. 

The problem of defining rough objects that permit reasoning about both intensional and extensional aspects posed in \cite{mkcfi2016} corresponds to identification of suitable semantic domains. Explicit perspectives, as for example in \cite{yy2015,gcd2018}, correspond similarly. Other semantic domains, including hybrid semantic domains, can be built from more complicated objects such as maximal antichains of mutually discernible objects \cite{am9114,am6999} (mentioned earlier), or even over the power set of the set of possible order-compatible partitions of the set of roughly equivalent elements \cite{am105,am501}. In fact, any general rough or soft reasoning context may be associated with a number of semantic domains \cite{am240,am9114,am9501,am501,am5586,am3930}. 
These are sometimes vaguely referred to as meta levels in the AIML literature. For example, type-I and type-II fuzzy sets are read in terms of descriptions of functions from different meta levels.  
Thus the concept of \emph{semantic domains} in abstract model theory \cite{md} is analogous to the usage here (though formalization of domains can be objected to).  

Even when \emph{rough sets are formalized as well-formed formulas} in a fixed language they do not refer to the same domain of discourse. For example, \cite{bk3,bc2,am3,dueo2011} refer to semantics of classical rough sets from different perspectives with that of \cite{am3} being a higher order semantics of the ability of objects to approximate -- this is not expressible in the others. 

For semantic domains in the context of fuzzy sets, the reader is referred to \cite{oep05}. In p46, the author essentially points out that possible solutions of the following problem (of subjective probability arguably) depend on the semantic domain used: 
\begin{quote}
A box contains ten balls of various sizes. Several of these are large and a
few are small. What is the probability that a ball drawn at random is neither
large nor small?
\end{quote}
Implicit in the problem is that subjective perceptions determine the being of objects. From a rough perspective, information about attributes of the objects would be necessary.

From the above scenario, it can be expected that possible concepts of rationality of approximations are additional impositions on the the domain of discourse. Current practices in rough sets do not explicitly require approximations to be rational in most cases; however, many conditions on possible decisions can be imposed. This lack of universality across domains further makes it natural to explore concepts of rationality at the object level. Apparently, the most convenient way is to define reasonable concepts through ideas of \emph{being a substantial part of} that are related to rough approximations and rationality.

Rationality is additionally approached through formalizations of principles in the context of modal or nonmonotonic logics with epistemological concerns. This research does not fit into those perspectives because of the relatively weaker assumptions on approximation operators.

Human-reasoning is often not about truth of statements or grades of truth of statements. Valuation of statements with grades of truth in Boolean ideas of $\{F, T\}$ or hypercubes generated by $[0, 1]$ are often meaningless and never intended. They are however used in soft decision-making with no proper validation or non-transparent expert approved mechanisms as in the LSP (Logic Scoring Preference) method \cite{jd2018}. Related assumptions do not provide a good basis for rationality.

\subsection{Dependence and its Semantic Domains}

If the existence of an object or process $A$ is causally implied by the existence of a collection of objects or processes $B$, then $A$ \emph{depends on} $B$ --this is the basic idea of dependence used in most contexts (including the study of databases). Multiple theories or perspectives of dependence are relevant in this research as the focus is not just on information tables and databases. Semantic specifications on the latter are often implemented through formal implication statements called \emph{first order dependencies} (that correspond to  some admissible rules). An overview of such practices can be found in \cite{famo86} (more details can be found in \cite{ewa1998,joanna2008,am501,am5586}). Interpretation of related formulas generally need to take the domains associated with attributes into account.
Formally, over a first order language with individual variables $\{a, b, c, \ldots$ (for entries in a relational database), logical symbols, connectives, constants and connectives being $\neg, \vee, \wedge, \longrightarrow, \forall, \exists$, predicate symbols $=, P, Q, R, \ldots $ atomic formulas are relational formula of the form $Qx_1\ldots x_n$ or equalities of the form $a=b$. In this setting, first-order dependencies can be defined by formulas of the form  (with $\Phi_i$ being relational formulas and $\Psi_i$ being atomic formulas):
\[(\forall x_1, \ldots x_r) ((\wedge_{i=1}^{n} \Phi_i) \longrightarrow (\exists z_1\ldots z_k) (\wedge_{i=1}^{r} \Psi_i)  ) \]

Dependencies can additionally be formulated in terms of indiscernability relations (see \cite{ewa1998}), and it is possible to express related ideas in those terms. The restricted relational calculus is another formal system that may be used alternatively.

Often higher constructions are implemented over the information tables or databases using relatively external heuristics. This leads to dependencies between sets of objects or sets of subsets of attributes. We will refer to these too as \emph{dependencies} and further qualify them. For example, the degree of inclusion function $i:\, \wp_{f}(S) \times \wp_{f}(S) \longmapsto [0, 1]$ is defined over pairs of finite subsets of a set $S$ ($\#()$ being the cardinality function) by \[i(A, B) = \dfrac{\# (A\cap B)}{\#(A)}\] It can be interpreted as the degree of inclusion of $A$ in $B$, and can additionally be interpreted as a subjective conditional probability among many other interpretations. For $i()$ to be a conditional probability function, it is necessary to specify a $\sigma$-algebra on $S$.

An axiomatic approach to probability theory based on probabilistic dependence is proposed in \cite{bd2010}, while the approach is generalized to set-valued probabilistic dependence in \cite{am3930,am9411} by the first author. The simpler measure-theoretic approach starts from a probability space $\mathbb{X} = \left( X,\,\mathcal{S} ,\, p  \right)$ with $X$ being a set, $\mathcal{S}$ being a $\sigma$-algebra over $X$ and $p$ a probability function. A dependence function over $\mathbb{X}$ is a function $\delta :\, {\mathcal{S}}^{2}\,\longmapsto \, \Re $  defined by 
\begin{equation} 
\delta (x, \, y)\,=\, p(x\cap y)\,-\, p(x)\,\cdot \, p(y)                                                                                                                                                                                                                                                         \end{equation}
The properties of $\delta$ can be used redefine or generalize the concept of probability spaces.

The interpretation of functions of the form $i$ as subjective probabilities or probabilities can therefore be related to concrete measures of dependence (one way or the other).

\section{Mereology}

Mereology \cite{ham2017} consists of a number of theoretical and philosophical approaches to relations of parthood (or \emph{is a part of} predicates) and relatable ones such as those of \emph{being connected to, being apart from, and being disconnected from}. Such relations can be found everywhere, and they relate to ontological features of any body of soft or hard knowledge (and their representation). These are not necessarily spatial, and can refer to various types of objects. In this section, the core assumptions used in this research are outlined. It builds on earlier work of the first author \cite{am240,am9969,am501,am5586} and others. 

The subject goes back to ancient times \cite{rgac15,cam19,ham2017}, and all cultures around the world have contributed to it from different perspectives. Formal approaches to mereology; however, started only at the end of the nineteenth century (though applications to STEM disciplines can be traced to the seventeenth century). The formal language used for a specific version of mereology depends on the domain of discourse (or semantic domain), and therefore a reasonable specification of the latter in natural language is essential for a proper understanding. For example, the wheels of a car maybe regarded as part of the latter, however again they can be regarded otherwise as the wheels lack the property of being a car. Here ontology can be used to regulate admissible meanings.

In the literature various types of mereologies \cite{ham2017,cav21,am240,psi,ldk,ng77,rjle2007} are known. The differences can be about axioms (when a common formal language is possible) or domains of discourse. \cite{cav21} imposes a common formal language on a number of approaches. Such reductionism is additionally evident in the \emph{Lesniewskian ontology-inspired rough mereology} of \cite{lp4,pls} where even the perspective of \cite{gk} (that \emph{theorems of ontology are those that are true in every model for atomic Boolean algebras without a null element}) is accepted. 
This makes the resulting model unsuitable for modeling human reasoning, though it appears to work for simpler robotic tasks, and such. Note that the relation of \emph{being roughly included to a degree $r$} is not transitive. Reasoning about vagueness in mereological perspectives requires one to add additional predicates almost always, and therefore associated axiomatics is involved \cite{lp2011,am240,am9969} and may need to be enhanced with ontologies.

As noted in \cite{am240}, \cite{ur} critically re-evaluates formal aspects of Lesniewskian mereology and points out a number of flaws in the reductionist approach of other authors. The difficulties in making the original formalism of Lesniewski with the more common languages of modern first order or second order logic is additionally highlighted there (\cite{ur}). The correct translation of expression in the  language of ontology to a common language of set theory (modulo simplifying assumptions) requires many axiom schemas and the converse translation (modulo simplifying assumptions) is doubtful (see pp 184-194,\cite{ur}). 

Another mereology that assumes many properties of the mereological predicates however formalized in an intuitionist perspective is \cite{intm2021}. The logics are compatible with the following extensionality principles too.
\begin{description}
 \item[Extensionality of Parthood]{If $a$ and $b$ are part of each other, then $a = b$} 
 \item[Extensionality of Overlap]{If $a$ and $b$ overlap the same things, then $a=b$}
 \item[Extensionality of Proper Parthood]{If $a$ and $b$ are composite things with the same proper parts, then $a = b$}
\end{description}
These extensionality conditions can additionally fail in rough reasoning. Suppose $a$ is part of $b$ if and only if the lower approximation of $a$ is a subset of the lower approximation of $b$, then extensionality of parthood fails. If $a=b$ is replaced by $a\approx_l b$ ($\approx_l$ being the lower rough equality defined by $a\approx_l b$ if and only if $a^l = b^l$), then a weaker form of extensionality of parthood holds. 

In a rough domain of reasoning, not every object of the classical domain is \emph{actualized}, that is available for reasoning because only approximations of objects can be perceived. This, in addition, makes general rough reasoning compatible with theories of semi sets (see \cite{am240,vh,vp9}) that is a framework for handling uncertainty and vagueness. 

The main difference of our approach as in \cite{am240} with the Polkowski-Skowron style mereological approach are that the basic requirements on part of relation is assumed to be minimal, and numeric valuation functions are not assumed. Rough inclusion functions are necessary for the basic concept of an object $x$ being an ingredient of an object $Z$ to a degree $r$ in \cite{ps3}. Here concepts of degree are expected to be replaceable by suitable algebraic structures (though simplifying assumptions may be needed in some applications). Apart from the motivations and reasons explained in \cite{am240,am9969,am501} by the first author, the need for formalizing rationality and flexible languages for teaching and educational research provides additional motivations. The latter is explained in brief by her in \cite{am2022c}. 

Basically, in classroom teaching of a subject like mathematics (and especially from a student-centric perspective \cite{am5559,jrp2016}), it is useful to use a language that is easy to understand, and at the same time be formalizable (possibly in a number of ways). In \cite{am2022c}, the use of mereology is suggested as a universal solution to the problem. It may be noted that partial formalizations of negative valuations in multisets \cite{fkqg2020} are related.   

Given a universe $S$, let $\neg S = \{\neg x: x\in S\}$ be another universe formed from the elements of $S$ and a negation operator symbol (not necessarily interpreted within $S$ in the classical sense). In \cite{vms09}, paraconsistent ideas of membership from $S$ to $\wp(S\times \neg S)$ are studied relative to truth valuations in the set $\mathcal{B} = \{\mathbf{t,\, f,\, i,\, u}\}$ (corresponding respectively to truth, falsity, inconsistency, and undeterminability) endowed with a \emph{knowledge} and Belnap's \emph{truth} ordering. The knowledge ordering on $\mathcal{B}$ is given  defined by $\mathbf{u\leq t \leq i}$ and $\mathbf{u\leq f \leq i}$ (this differs from Belnap's truth ordering given by $\mathbf{f\leq i \leq t}$ and $\mathbf{f\leq u \leq t}$). Membership is then of the following four types:

\begin{equation*}
x\epsilon A  = \left\lbrace  \begin{array}{lll}
 \mathbf{t} & \text{if } x\in A\, &\,(\neg x) \notin  A\\
 \mathbf{f} & \text{if }  x\notin A\, &\, (\neg x) \in A\\
 \mathbf{u} & \text{if }  x\notin A\, &\,(\neg x) \notin  A\\
 \mathbf{i} & \text{if } x\in A\, &\,(\neg x) \in  A
 \end{array} \right. \tag{$\epsilon$} 
\end{equation*}
The full mereological import of this is not discussed in \cite{vms09}; however, a number of parthoods are defined and studied in the paper. Connections with \cite{am9969} will appear separately.

In data science, a much restricted version of the formalizability problem of education research mentioned above is that of formally describing information flow in distributed systems \cite{zpsk07,base97,ssmc2019}. Apart from rough sets, other approaches are known (chapter-7 of \cite{mkcs} has a discussion of ones using graded fuzzy sets and variants) - a key restriction in these situations is the existence of infomorphisms that are means of interpretation across systems. Use of paraconsistent sets (with extended ideas of membership and parthood) is an essential part of the proposed solution in it.

\subsection{Assumptions}

In general rough contexts involving any number of approximation operators, it is possible to speak of definite and rough objects in multiple senses that arise from preferred concepts of approximate equality and approximations. Even situations in which the goal is to look for possible explanations that fit approximations specified by agents without explanation, can be handled by the general frameworks of granular operator spaces/partial algebras introduced by the first author. 

\emph{The frameworks are built with the intent to prove results with as few assumptions (axioms) as is possible}. However in the contexts of general rough sets multiple general part of relations arise from the relation of rough objects to other rough objects, rough objects to other objects, and classical objects to others. These can be expressed with two kinds of part of relations and  is the approach adopted in earlier papers by the first author \cite{am9969}. Multiple part-of relations are additionally suggested in \cite{rjle2007}.  

At least three sets of mereological predicates ($\pc$, $\leq$ and $\pc_s$) are assumed in the present research. While the conditions satisfied by $leq$ or alternatively in terms of the partial operations $\vee$ and $\wedge$ axioms refer to a generalizations of $\subseteq$ on power sets of the universe in the classical case, $\pc$ is intended for parthoods that mainly concern rough objects, and $\pc_s$ for expressing substantial parthood. 

The partial operations $\vee$ and $\wedge$ are required to be weakly commutative, distributive and satisfy absorption. For example, if $H$ is a set of attributes, and $\mathbb{S}$ is a subset of the power set $\wp (H)$, then the set-theoretic operations $\cup$ and $\cap$ would be interpretible on $\mathbb{H}$ as idempotent partial operations. The actual axioms \textsf{(G1-G5)} satisfied by $\vee$ and $\wedge$ are permitted to be quite weak (but it is always possible to add additional conditions if required). 

\begin{align*}
(\forall a, b) a\vee b \stackrel{\omega}{=} b\vee a \;;\; (\forall a, b) a\wedge b \stackrel{\omega}{=} b\wedge a \tag{G1}\\
(\forall a, b) (a\vee b) \wedge a \stackrel{\omega}{=} a \; ;\; (\forall a, b) (a\wedge b) \vee a \stackrel{\omega}{=} a \tag{G2}\\
(\forall a, b, c) (a\wedge b) \vee c \stackrel{\omega}{=} (a\vee c) \wedge (b\vee c) \tag{G3}\\
(\forall a, b, c) (a\vee b) \wedge c \stackrel{\omega}{=} (a\wedge c) \vee  (b\wedge c) \tag{G4}\\
(\forall a, b) (a\leq b \leftrightarrow a\vee b = b \,\leftrightarrow\, a\wedge b = a  ) \tag{G5}
\end{align*}

Because multi-sets with finite number of repetitions of elements can always be represented as sets in ZFC (at a relatively higher level of abstraction as functions from a set into positive integers) \cite{wbl1991}, they can be accommodated as well. Variable precision $L$-fuzzy rough sets (for a partially ordered set $L$ or even a residuated lattice) can additionally be handled in the frameworks provided the approximations are granular as functions and relations are essentially sets. This will appear separately.

To cover a broader domain (especially to accommodate directed graphs), it is necessary to remove conditions \textsf{G3} and \textsf{G5} and permit idempotence \textsf{I1} and \textsf{I2} of $\vee$ and $\wedge$. This would additionally include all weakly associative lattices (and tournaments), partial weakly associative lattices, lambda lattices \cite{sva,am105} and certain directed graphs \cite{frgra1976}. \emph{Thus even from a purely rough perspective, a number of nonequivalent models that are minimal with respect to the axioms satisfied are necessary}. 

Some conditions that are expected to be satisfied by the parthood $\pc$ in the granular models are 

\begin{align*}
(\forall x) \pc xx \tag{PT1}\\
(\forall x, b) (\pc xb \, \&\, \pc bx \longrightarrow x = b) \tag{PT2}\\
(\forall a \in \mathbb{S})\,  \pc a^l  a\,\&\,a^{ll}\, =\,a^l \,\&\, \pc a^{u}  a^{uu}  \tag{UL1}\\
(\forall a, b \in \mathbb{S}) (\pc a b \longrightarrow \pc a^l b^l \,\&\,\pc a^u  b^u) \tag{UL2}\\
\bot^l\, =\, \bot \,\&\, \bot^u\, =\, \bot \,\&\, \pc \top^{l} \top \,\&\,  \pc \top^{u} \top  \tag{UL3}\\
(\forall a \in \mathbb{S})\, \pc \bot a \,\&\, \pc a \top    \tag{TB}
\end{align*}

The parthood $\pc$ is constrained by the crisp reflexivity condition \textsf{PT1}, possibly by antisymmetry \textsf{PT2}, tops and bottoms \textsf{TB} and conditions connected with approximations \textsf{UL1, UL2} and \textsf{UL3}. Because it is known that $\pc a^l a$ can fail to hold in some rough semantics (in particular, when approximations are defined in irreflexive general approximation spaces \cite{am24}), \textsf{UL1} is relaxed in a Pre* GGS (defined in the next section). \textsf{PT2} is typical of reasoning in the rough set way (general), and is a weaker version of the mereological principle(s): \emph{things with the same parts are equal/identical}.

Suppose a word $X$ is used in senses $s_1, s_2, \ldots s_n $, and a sense $s$ is shared by many as a prescriptive semantic function of the meaning of $X$, then a new semantic definition that   
refers $s$ as a lower approximation of the meaning of $X$ may appear to be rational. At the same time, it is a substantial part of the meaning of $X$ (but in a perspective).

Axiomatic concepts of \emph{apparent parthood} are introduced and studied by the first author in \cite{am9969} -- it is motivated by the need to model partly warranted claims in specific contexts. For example, there are a number of reasons for associating problems with a person's digestive system with cardiac conditions. However it is not that (from a statistical perspective) digestive problems cause cardiac problems always. When both problems are present, then it is reasonable to assert that the former is \emph{apparently part of} the latter. Similarly anxiety can be \emph{apparently part of} digestive problems; however, it does not follow that the former is \emph{apparently part of} cardiac problems. 

Fuzzy and degree valuations of these perceptions are even less justified when they are based on ontology-free arbitrary approximate judgments. 

\subsection{Granules and Granulations}

A granule may be vaguely defined as some concrete or abstract realization of relatively simpler (or crisper) objects through the use of which more complex problems may be solved. They exist relative to the problem being solved in question, and can be specified in different non-equivalent ways. For example, they can be specified by the internal attributes of objects, precision levels of possible solutions, or precision levels attained by objects. The axiomatic approach to granular computing is invented the first author in \cite{am240} based on earlier work in rough semantics. Further improvements by her can be found in \cite{am5586,am501,am3690}. The differences with primitive granular computing and the classical granular computing (typically involving numeric precision values) \cite{tyl,lz9,tyl1994,gll2006,yyi1996,ya01,yy5} are additionally explained in these. In \cite{am5586} and \cite{am501}, it is actually argued that the latter can be traced to algorithms in ancient mathematics.

In the axiomatic frameworks (AGCP) \cite{am501,am240} that does not refer to numeric precision for defining granules, the problem of defining or rather extracting concepts that qualify require much work in the specification of semantic domains and process abstraction. Granulations are collections of all granules and is denoted by $\mathcal{G}$. For a granulation to be \emph{admissible}, it is required that every approximation is term-representable by granules, that every granule in $\mathcal{G}$ coincides with its lower approximation (granules are lower definite), and that all pairs of distinct granules are part of definite objects (those that coincide with their own lower and upper approximations). The formal axioms are in the following section.

\section{Granular Frameworks}

Aspects of the mereology based axiomatic approach to granularity due to the first author in \cite{am240,am5586,am501} are explained in previous sections, and essential notions are repeated for convenience. In a \emph{high general granular operator space} (\textsf{GGS}), defined below, aggregation and co-aggregation operations ($\vee, \,\wedge$) are conceptually separated from the binary parthood ($\pc$), and a basic partial order relation ($\leq$). Parthood is assumed to be reflexive and antisymmetric. It may satisfy additional generalized transitivity conditions in many contexts. Real-life information processing often involves many non-evaluated instances of aggregations (fusions), commonalities (conjunctions) and implications because of laziness or supporting meta data or for other reasons  -- this justifies the use of partial operations. Specific versions of a \textsf{GGS} and granular operator spaces have been studied in \cite{am501}. Partial operations in \textsf{GGS} permit easier handling of adaptive granules \cite{skajsd2016} through morphisms. The universe $\underline{\mathbb{S}}$ may be a set of collections of attributes, labeled or unlabeled objects among other things.

\begin{definition}\label{gfsg}
A \emph{High General Granular Operator Space} (\textsf{GGS}) $\mathbb{S}$ is a partial algebraic system  of the form $\mathbb{S} \, =\, \left\langle \underline{\mathbb{S}}, \gamma, l , u, \pc, \leq , \vee,  \wedge, \bot, \top \right\rangle$ with $\underline{\mathbb{S}}$ being a set, $\gamma$ being a unary predicate that determines $\mathcal{G}$ (by the condition $\gamma x$ if and only if $x\in \mathcal{G}$) 
an \emph{admissible granulation}(defined below) for $\mathbb{S}$ and $l, u$ being operators $:\underline{\mathbb{S}}\longmapsto \underline{\mathbb{S}}$ satisfying the following ($\underline{\mathbb{S}}$ is replaced with $\mathbb{S}$ if clear from the context. $\vee$ and $\wedge$ are idempotent partial operations and $\pc$ is a binary predicate. Further $\gamma x$ will be replaced by $x \in \mathcal{G}$ for convenience.):
\begin{align*}
(\forall x) \pc xx \tag{PT1}\\
(\forall x, b) (\pc xb \, \&\, \pc bx \longrightarrow x = b) \tag{PT2}\\
(\forall a, b) a\vee b \stackrel{\omega}{=} b\vee a \;;\; (\forall a, b) a\wedge b \stackrel{\omega}{=} b\wedge a \tag{G1}\\
(\forall a, b) (a\vee b) \wedge a \stackrel{\omega}{=} a \; ;\; (\forall a, b) (a\wedge b) \vee a \stackrel{\omega}{=} a \tag{G2}\\
(\forall a, b, c) (a\wedge b) \vee c \stackrel{\omega}{=} (a\vee c) \wedge (b\vee c) \tag{G3}\\
(\forall a, b, c) (a\vee b) \wedge c \stackrel{\omega}{=} (a\wedge c) \vee  (b\wedge c) \tag{G4}\\
(\forall a, b) (a\leq b \leftrightarrow a\vee b = b \,\leftrightarrow\, a\wedge b = a  ) \tag{G5}\\
(\forall a \in \mathbb{S})\,  \pc a^l  a\,\&\,a^{ll}\, =\,a^l \,\&\, \pc a^{u}  a^{uu}  \tag{UL1}\\
(\forall a, b \in \mathbb{S}) (\pc a b \longrightarrow \pc a^l b^l \,\&\,\pc a^u  b^u) \tag{UL2}\\
\bot^l\, =\, \bot \,\&\, \bot^u\, =\, \bot \,\&\, \pc \top^{l} \top \,\&\,  \pc \top^{u} \top  \tag{UL3}\\
(\forall a \in \mathbb{S})\, \pc \bot a \,\&\, \pc a \top    \tag{TB}
\end{align*}
Let $\pp$ stand for proper parthood, defined via $\pp ab$ if and only if $\pc ab \,\&\,\neg \pc ba$). A granulation is said to be admissible if there exists a term operation $t$ formed from the weak lattice operations such that the following three conditions hold:
\begin{align*}
(\forall x \exists
x_{1},\ldots x_{r}\in \mathcal{G})\, t(x_{1},\,x_{2}, \ldots \,x_{r})=x^{l} \\
\tag{Weak RA, WRA} \mathrm{and}\: (\forall x\, \exists
x_{1},\,\ldots\,x_{r}\in \mathcal{G})\,t(x_{1},\,x_{2}, \ldots \,x_{r}) =
x^{u},\\
\tag{Lower Stability, LS}{(\forall a \in
\mathcal{G})(\forall {x\in \underline{\mathbb{S}}) })\, ( \pc ax\,\longrightarrow\, \pc ax^{l}),}\\
\tag{Full Underlap, FU}{(\forall
x,\,a \in\mathcal{G} \exists
z\in \underline{\mathbb{S}}) \, \pp xz,\,\&\,\pp az\,\&\,z^{l}\, =\,z^{u}\, =\,z,}
\end{align*}
\end{definition}

\begin{definition}
\begin{itemize}
\item {In the above definition, if the antisymmetry condition \textsf{PT2} is\\ dropped, then the resulting system will be referred to as a \emph{Pre-GGS}. If the restriction $\pc a^l  a$  is removed from \textsf{UL1} of a \textsf{pre-GGS}, then it will be referred to as a \emph{Pre*-GGS}.}
\item {In a \textsf{GGS} (resp \textsf{Pre*-GGS}), if the parthood is defined by $\pc ab$ if and only if $a \leq b$ then the \textsf{GGS} is said to be a \emph{high granular operator space} \textsf{GS} (resp. \textsf{Pre*-GS)}.}
\item {A \emph{higher granular operator space} (\textsf{HGOS}) (resp \textsf{Pre*-HGOS}) $\mathbb{S}$ is a \textsf{GS} (resp \textsf{Pre*-GS}) in which the lattice operations are total.}
\item {In a higher granular operator space, if the lattice operations are set theoretic union and intersection, then the \textsf{HGOS} (resp. Pre*-HGOS) will be said to be a \emph{set HGOS} (resp. \emph{set Pre*-HGOS} ). In this case, $\underline{\mathbb{S}}$ is a subset of a powerset, and the partial algebraic system reduces to $\mathbb{S} \, =\, \left\langle \underline{\mathbb{S}}, \gamma, l , u, \subseteq , \cup,  \cap, \bot, \top \right\rangle$ with $\underline{\mathbb{S}}$ being a set, $\gamma$ being a unary predicate that determines $\mathcal{G}$ (by the condition $\gamma x$ if and only if $x\in \mathcal{G}$). Closure under complementation is not guaranteed in it. }
\end{itemize}
\end{definition}

In \cite{am5550}, it is shown that the binary predicates can be replaced by partial two-place operations and $\gamma$ is replaceable by a total unary operation. This results in a semantically equivalent partial algebra called a \emph{high granular operator partial algebra} (GGSp). It extends to the other generalizations mentioned.

\section{General Rough Inclusion Functions}\label{grif}

Relative to form, rough inclusion functions are known in the literature by several other names such as conditional subjective probability, relative degree of misclassification, majority inclusion function, and inclusion degree. However the interpretation of these functions have ontologies associated. An algebraic study of generalizations of such functions by the first author can be found in a forthcoming paper \cite{am111111}. In this section, some of the definitions and properties are easily  extended to Pre*-GGS with additional comments for the purposes of this paper. Further additional axioms are introduced for ensuring better properties of generalized VPRS approximations. It may be noted that decision-theoretic rough sets \cite{xwzl2013,yy3w2011} additionally employ subjective probabilistic variants of the rough inclusion functions.

Intuitively, generalizations of rough inclusion functions are likely to work perfectly when most of the conditions \textsf{A1-A5} actually hold. These are not independent of each other.

\textbf{[A1]} \emph{The contribution of attributes in the construction of approximation can be assigned numeric weights across distinct instances of usage}. This needs to make sense in the first place for the subjective Bayesian interpretation suggested in the literature as in \cite{sl2008,sl2006} 
to be admissible. For example, weighting different attributes of U.S. presidents to evaluate their war-mongering tendency over and above the designs of the US imperialist foreign policy can be a meaningless exercise relative to the belief that big corporations determine the power structures and wars.

\textsf{[A2]} \emph{The contribution of attributes to approximations have uniform weightage across approximations}. On a closer examination, this principle is complicated by the factors and processes involved in interpreting it. If attributes $B, B_1,$ and $B_2$ are part of the lower approximation of $w$, and $B, B_2, B_3$ and $B_4$ are part of $v$, then the absolute contribution of $B$ is the same in both approximations. It does not depend on that of other $B_i$s.

\textbf{[A3]} \emph{The functions are robust and sensitive relatively}. That is the value of the function does not change much with small deviations of its arguments \cite{skajsd2016}, and is stable relative to the context. For example, if two subsets are similar, then their degree of inclusion in a third set should additionally be almost equal relative to the perspective of the context.

\textbf{[A4]} \emph{Every aggregate of attributes is meaningful}. This is required because the functions are assumed to be total (and not partial).

\textbf{[A5]} \emph{Attributes are mutually pairwise independent in a appropriate sense relative to the goals of computation}. This means that superfluous attributes (relative to approximations) are already removed. Note that sets of attributes as opposed to individual attributes are considered as arguments of the functions.

The ideas of robustness and stability are always relative to a finite number of purposes or use cases in application contexts. \textsf{A4} means that meaningful restrictions on forming sets of attributes can lead to the failure of certain axioms, even though it may result in lesser  computational load of algorithms. Failure of \textsf{A5} in real applications is common, and the damage caused by assuming otherwise can possibly be mitigated by generalized RIFs that satisfy weaker conditions. 

In this section, the different known rough inclusion functions are generalized to \textsf{pre*-GGS} of the form $\mathbb{S} \, =\, \left\langle \underline{\mathbb{S}}, \mathcal{G}, l , u, \pc, \vee,  \wedge, \bot, \top \right\rangle$. If $\kap : \underline{\mathbb{S}}^2 \longmapsto [0, 1]$ is a map, consider the conditions ($\delta \in [0, 1]$ is assumed to be fixed),
\begin{align*}
(\forall a)\, \kap (a, a) = 1    \tag{U1}\\
(\forall a, b)(\kap (a, b) = 1 \leftrightarrow   \pc a b) \tag{R1}\\
(\forall a, b, c)(\kap (b, c) = 1 \longrightarrow \kap (a, b) \leq \kap (a, c))   \tag{R2}\\
(\forall a, b, c)(\pc bc \longrightarrow \kap (a, b) \leq \kap (a, c))   \tag{R3}\\
(\forall a, b)(\pc ab \longrightarrow \kap (a, b) = 1)   \tag{R0}\\ 
(\forall a, b)(\kap (a, b) = 1\longrightarrow \pc ab )   \tag{IR0}\\ 
(\forall a) (\pp \bot a\longrightarrow  \kap (a, \bot ) = 0)   \tag{RB}\\
(\forall a, b)(\kap ( a, b)=0 \longrightarrow a \wedge b = \bot)  \tag{R4}\\
(\forall a, b)( a\wedge b = \bot \,\&\, \pp \bot a \longrightarrow \kap (a, b) = 0)   \tag{IR4}\\ 
(\forall a, b)(\kap ( a, b)=0 \,\&\, \pp \bot a  \leftrightarrow a \wedge b = \bot)  \tag{R5}\\
(\forall a, b, c)( \pp \bot a \, \&\, b\vee c = \top\longrightarrow  \kap (a , b) + \kap (a, c) =1)  \tag{R6}\\
(\forall a, b, c)(\pc ca \,\&\, \pc ab \, \& \, \kap (b, c)\geq \delta \longrightarrow \kap (a,c) \geq \delta)  \tag{RV}\\
(\forall a, b, c)(\kap (a, c) \geq \delta \,\&\, \kap (b, c)\geq \delta \,\&\, \pc c (a\wedge b) \longrightarrow \kap (a\wedge b,c) \geq \delta)  \tag{RI}
\end{align*}

These (except for \textsf{RV, and RI} are analogous extensions of the definition in \cite{ag2009,ag3}. The conditions with $\delta \in [0, 1]$ are intended for a specific value of the parameter with no additional restrictions. Therefore they need to be specified in actual usage (typically as $\alpha$ or $1-\alpha$.  $rif_3$ is RB, and $rif_{2*}$ is R2 in a set HGOS under the conditions mentioned. Proofs of the next theorem for set HGOS can be deduced from those for wqRIFs in \cite{ag2009}. These carry over to \textsf{HGOS} directly, while the proofs in a \textsf{GS} are not hard. \textsf{RV and RI} and the specific versions are new axioms that can be very useful in generalized VPRS contexts. These relate two instances of $\kap$ properties in which the argument on the right are preserved. It may be noted that the axioms \textsf{RV} happens often in  probabilistic rough sets and three-way decision-making \cite{yy10}.

\begin{proposition}
The classical RIF $\nu$ defined in the background section satisfies the condition $RV$ for all $\delta \in [0, 1] $. 
\end{proposition}
\begin{proof}
The proof follows from the definition. 
\end{proof}

\begin{theorem}
The following implications between the properties are provable.
\begin{description}
\item[prif1]{If $\mathbb{S}$ satisfies R1, then R3 and R2 are equivalent.}
\item[prif2]{R1 if and only if R0 and IR0 are satisfied.}
\item[prif3]{R0 and R2 imply R3.}
\item[prif4]{IR0 and R3 imply R2.}
\item[prif5]{IR4 implies RB.}
\item[prif6]{IR4 and R4 if and only if R5.}
\item[prif7]{When complementation is well defined in a set HGOS then R0 and R6 imply IR4.}
\item[prif8]{When complementation is well defined in a set HGOS then IR0 and R6 imply R4.}
\item[prif9]{When complementation is well defined in a set HGOS then R1 and R6 imply R5.}
\end{description}
Further both R1 and R0 imply U1.
\end{theorem}
\begin{proof}
Aspects of the proof are illustrated below
\begin{description}
\item[prif1]{Suppose $\pc bc$ for some $b, c\in \mathbb{S}$, then by R1 $\kap (b, c) =1$ and conversely. In R2 and R3 the premise can be interchanged in the conditional implication when R1 holds.}
\item[prif2]{Obvious.}
\item[prif3]{Suppose R0 and R2 hold. If $\pc bc$ for some $b, c\in \mathbb{S}$ then by R0 $\kap (b, c) =1$. Thus for any $a$ $\kap (a, b) \leq \kap (a, c)$. That is R3 follows from R0 and R2}
\item[prif5]{Substituting $\bot$ for $b$ in IR4 yields RB.}
\item[pref6]{Is obvious.}
\end{description}
\end{proof}

\begin{definition}
A map $\kap : (\mathbb{S})^2 \longmapsto [0, 1]$ on a \textsf{pre*-GGS} $\mathbb{S}$ is
\begin{itemize}
\item {a \emph{general rough inclusion} function (gRIF) if it satisfies R1 and R2,}
\item {a \emph{general quasi rough inclusion} function (qRIF) it it satisfies R0 and R2,}
\item {a \emph{general weak quasi rough inclusion} function (wqRIF) if it satisfies R0 and R3, and}
\item {a \emph{general para rough inclusion} function (pRIF) if it satisfies R0 and RV.}
\end{itemize}
\end{definition}

\begin{proposition}
In a \textsf{GS} $\mathbb{S}$, every \textsf{gRIF} is a \textsf{qRIF} and every \textsf{qRIF} is a \textsf{wqRIF}.  
\end{proposition}

\subsection{Examples: }

\paragraph{Other Generalizations:}
The above generalizations can additionally be done relative to the $\leq$ order of a GGS (instead of $\pc$). This leads to a focus on partial aggregation and commonality operations, that may not satisfy nice properties relative to the approximations. 

\subsection{Specific Weak Quasi-RIFs}

gRIFs and variants thereof are defined over power sets in \cite{ss2010,ag2009}. For rewriting them in the set HGOS way, it is necessary to assume that $\underline{\mathbb{S}} = \wp(\top)$, $\top$ being a finite set, $\bot = \emptyset$,  $\pc = \leq = \subseteq$, $\vee = \cup$ and $\wedge = \cap$. Specifically, the following functions have been studied in \cite{ag2009} and have been used to define concepts of approximation spaces.   

\begin{equation*}\label{k1}
\nu_1(A, B) = \left\lbrace  \begin{array}{ll}
 \dfrac{\# (B)}{\# (A\cup B)} & \text{if } A\cup B\neq \emptyset\\
 1 & \text{otherwise}\\
 \end{array} \right. \tag{K1}                                                                                                             
\end{equation*}

\begin{equation*}\label{k2}
\nu_2(A, B) =  \dfrac{\# (A^c \cup B)}{\# (\top)}  \tag{K2}                                                                                                             
\end{equation*}

If $0 \leq s < t \leq 1$, and $\nu : {\mathbb{S}}^2 \longmapsto [0, 1]$ is a \textsf{RIF}, then let $\nu_{s,t}^{\nu} : {\mathbb{S}}^2 \longmapsto [0, 1]$ be a function defined by 

\begin{equation*}\label{kst}
\nu^{\nu}_{s,t}(A, B) = \left\lbrace  \begin{array}{ll}
0 & \text{ if } \nu(A, B) \leq s \\
 \dfrac{\nu (A, B) - s}{t - s} & \text{ if } s < \nu(A, B) < t,\\
 1 & \text{ if } \nu(A, B ) \geq t\\
 \end{array} \right. \tag{Kst}                                                                                                             
\end{equation*}

\begin{proposition}{\cite{ag2009}} 
In general, $\nu^{\nu}_{s,t}$ is a weak quasi \textsf{RIF} and $\nu^{\nu}_{s,1}$ is a quasi \textsf{RIF}. 
\end{proposition}

\begin{proposition}
\begin{enumerate}
\item {$\nu$ satisfies \textsf{RV} and \textsf{RI}}
\item {$\nu_{s, t}\nu$ satisfies \textsf{RV} and \textsf{RI}}
\end{enumerate}
\end{proposition}

\begin{proof}
$\nu$ satisfies \textsf{RV}.
\begin{enumerate}
\item {If for any $a, b$ and $c$, and a $\alpha \in [0, 1)$, $c\subseteq a \subseteq b$ and $\nu (b, c)  \geq \alpha$ (note that $\pc$ is transitive here),}
\item {then if $\nu(b,c) <1 $  $\frac{\# (c\cap b)}{\# (b)} \geq \alpha$. }
\item {It follows that if $\frac{\# (c\cap a)}{\# (a)} <1 $ $\frac{\# (c\cap a)}{\# (a)} \geq \frac{\# (c\cap b)}{\# (b)} \geq \alpha$ (because of the possible sizes of the numerators and denominators).}
\item {When $\alpha =1$, too the implication holds because $\#(b)= 0$ implies the other two sets are empty. Thus $\nu$ satisfies \textsf{RV}.}
\end{enumerate}

$\nu$ satisfies \textsf{RI}.
\begin{enumerate}
\item {If for any $a, b$ and $c$, and a $\alpha \in [0, 1)$, $\nu (a, c)  \geq \alpha$, $\nu (b, c)\geq \alpha$, and $\pc c (b\cap a)$}
\item {then $\alpha \leq \frac{\# (a\cap c)}{\# (a)} \leq 1 $ and $\alpha \leq \frac{\# (b\cap c)}{\# (b)} \leq 1 $ }
\item {$\pc c (b\cap a)$ implies $c\subseteq a$ and $c\subseteq b$, $\#(a\cap b\cap c) = \# (a\cap c) = \# (b\cap c)$ }
\item {Further, $\#(a\cap b) \leq \#(a)$ and $\#(a\cap b) \leq \#(b)$, ensures the conclusion in all cases.}
\end{enumerate}

The proof for $\nu^{\nu}_{s,t}$ is a case based extension of the above.
\end{proof}

\begin{remark}
In \textsf{RI}, if the condition $\pc c (a\wedge b)$ is removed, then both $\nu$ and $\nu^\nu_{s,t}$ fail to satisfy it. This can be checked by taking $a = \{1, 2, 3, 6\}, b = \{3, 5, 7, 8, 9\}$,  $c = \{2, 5, 6\}$, and $\alpha = \frac{1}{5}$
\end{remark}

\subsection{Granular Rough Inclusion Functions}\label{granrif}

In application contexts, one typically reasons over approximate information that maybe modeled as approximations. The exact origin or process or even their correct description may not be known to the scientists. When we are studying images derived through magnetic resonance or other techniques, the ground knowledge of practitioners is typically limited. The exact processes involved in the disease may not be sufficiently understood, and the scope for discovery remains fertile. It is possible that the things sought from the images are approximations that are not good enough for the undiscovered processes. This suggests that it would be more reasonable to construct generalized RIFs in the context with approximations. This realization of the need to avoid contamination motivates the concept of granular RIFs defined below. \emph{Every concept of a weak quasi rough inclusion function considered earlier is contaminated relative to this perspective} \cite{am9969,am5586}. 

Over a \textsf{set HGOS}, the following conceptual variants of rough inclusion functions can be defined. 

\begin{definition}\label{grifhgos}
If $\mathbb{S}$ is a \textsf{set HGOS}, $A, B\in \mathbb{S}$, $\sigma, \pi \in \{l, u\}$  and the denominators in the expression is non zero, then let 
\begin{equation*}
\nu_{\sigma \pi}(A, B ) = \dfrac{\#(A^\sigma \cap B^\pi )}{\#(A^\sigma )}  \tag{$\sigma\pi$-grif1} 
\end{equation*}
If $\# (A^\sigma) = 0 $, then set the value of $\nu_{\sigma \pi}(A, B)$ to $1$. 
$\nu_{\sigma \pi}$ will be said to be a \emph{basic granular rough inclusion function} (\textsf{bGRIF}).
\end{definition}

In addition, variants of this definition are of interest:

\begin{definition}\label{cogrifhgos}
If $\mathbb{S}$ is a \textsf{set HGOS}, $A, B\in \mathbb{S}$, $\sigma, \pi \in \{l, u\}$  and the denominators in the expression is non zero, then let 
\begin{equation*}
\nu_{\sigma \pi}(A, B ) = \dfrac{\#(A^\sigma \cap B^\pi )}{\#(A^\pi )}  \tag{$\sigma\pi$-grif2} 
\end{equation*}
If $\# (A^\pi) = 0 $, then set the value of $\nu_{\sigma \pi}(A, B)$ to $1$. 
$\nu_{\sigma \pi}$ will be said to be a \emph{cobasic granular rough inclusion function} (\textsf{cGRIF}).
\end{definition}

\begin{proposition}
In the context of Definition \ref{grifhgos}, \[(\forall A, B\in \mathbb{S} ) \,0\leq  \nu_{\sigma \pi} (A, B) \leq 1\] This proposition does not hold for cobasic granular rough inclusion functions in general.
\end{proposition}

\begin{theorem}
In a \textsf{set HGOS} $\mathbb{S}$ with $\bot=\emptyset$, all of the following hold ($\alpha$ being any one of $ll, lu, ul$ or $uu$):
\begin{align*}
(\forall A, B)\, \nu_{ul}(A, B ) \leq \nu_{uu}(A, B )  \tag{ulu2}\\
(\forall A, B)\, \nu_{ll}(A, B ) \leq \nu_{lu}(A, B )  \tag{llu2}\\ 
(\forall A, B, E)\,(B\subset E \longrightarrow \, \nu_{\alpha}(A, B) \leq \nu_{\alpha}(A, E) )   \tag{mo}\\
(\forall A )\,\nu_{lu}(A, A) \leq \nu_{ll}(A, A) = 1 = \nu_{uu}(A, A)    \tag{refl}\\
(\forall A)\,\nu_{\alpha}(\bot , A) =1   \tag{bot}\\
(\forall A)\, (\top=\top^l = \top^u \longrightarrow \nu_{\alpha}(A,\top) =1)   \tag{top}
\end{align*}
\end{theorem}
\begin{proof}
 \begin{enumerate}
\item {\textsf{ulu2} follows from $(\forall B \in \mathbb{S})\, B^l \subseteq B^u$.  }
\item {Proof of \textsf{llu2} is similar.}
\item {Since both $l$ and $u$ are monotonic and the denominator is invariant in\\ $\nu_{\alpha}(A, B) \leq \nu_{\alpha}(E, B)$, \textsf{mo} follows. }
\item {Proof of \textsf{refl} is direct.}
\item {Only in the condition $\bot$, is the assumption $\bot = \emptyset$ used.}
\end{enumerate}
\end{proof}

\section{Generalized VPRS and Variants}\label{gvprssu}

The basic idea underlying variable precision rough sets and variants is that \emph{an approximation is justified if it is determined by granules that have at least some numeric level of common elements with the approximated}. Further, the natural setting of related approximations is a pre*-set HGOS relative to the strong assumptions of VPRS (unless the subset relation is replaced by a more general parthood). It can be that the approximations are constructed by points and their neighborhoods -- this can directly fall under modal logic perspectives \cite{amk2018}.  
If a set $A$ is to be approximated at level $\alpha\in [0, 0.5]$ by a collection $\mathcal{F}$ (that may or may not be a granulation $\mathcal{G}$ or a set of definite elements $\delta$ (relative to the usual approximation)), then consider the following definitions ($\kap$ being a wqRIF):
\begin{align*}
B^{l_\alpha} = \bigcup \{h:\, h\in \mathcal{F}\& h\subseteq B \, \&\,\kap (B,h)\geq 1-\alpha \}   \tag{l-vprs}\\
B^{u_\alpha} = \bigcup \{h:\, h\in \mathcal{F}\& h\cap B \neq \emptyset\, \&\,\kap (B,h) > \alpha \}   \tag{u-vprs}
\end{align*}

The positive, negative, and boundary regions associated with $B$ are defined as follows:
\begin{align*}
POS_{\alpha}(B) = B^{l_\alpha}   \tag{Positive}\\
NEG_{\alpha}(B) = \bigcup \{h:\, h\in \mathcal{F}\& h\cap B \neq \emptyset\, \&\,\kap (B,h)\leq \alpha \}   \tag{Negative}\\
BND_{\alpha}(B) = B^{u_{\alpha}}\setminus B^{l_{\alpha}}   \tag{Boundary}
\end{align*}

For convenience, unless stated otherwise, all of the approximations of this section will be assumed to be defined over an algebraic system (called \emph{rough convenience}) of the form \[\mathbb{S} \, =\, \left\langle \underline{\mathbb{S}}, \gamma, \subseteq , \cup,  \cap, \bot, \top \right\rangle\] with $\underline{\mathbb{S}}$ being a subset of a powerset $\wp (H)$, $\gamma$ being a unary predicate that determines $\mathcal{G}$ (by the condition $\gamma x$ if and only if $x\in \mathcal{G}$). It is assumed that $\mathbb{S}$ is closed under $\cup$ and $cap$. Note that this may be the essential the core of a \textsf{pre*-set HGOS} without the approximations and the admissibility conditions on $\mathcal{G}$. Set complementation may be defined.

\emph{The following theorem has surprising new cautious monotony properties \cite{dm94} ignored in all cases in the literature}.

\begin{theorem}\label{alpha}
Over a rough convenience $\mathbb{S}$, if $\kap$ is a wqRIF that satisfies \textsf{RV}, and $l_\alpha$, and $u_\alpha$ are as above then all of the following hold:
\begin{align}
(\forall a, b)(a^{l_\alpha}\subseteq b\subseteq a \longrightarrow a^{l_\alpha}\subseteq b^{l_\alpha})   \tag{lA-cmo}\\
(\forall a, b)(a\subseteq b\subseteq a^{u_\alpha}  \longrightarrow a^{u_\alpha}\subseteq b^{u_\alpha})   \tag{uA-cmo}\\
(\forall a, b) a^{l_\alpha}\cap b^{l_\alpha}\subseteq (a\cap b)^{l_\alpha}    \tag{lA-capc}\\
(\forall a) a^{l_\alpha} = a^{l_\alpha l_\alpha} \subseteq a   \tag{lA-idem}\\
(\forall a) a^{l_\alpha} \subseteq a^{u_\alpha} \tag{luA}\\
(\forall a) a^{l_\alpha} \subseteq a \tag{li}
\end{align}
\end{theorem}

\begin{proof}
wqRIFs satisfy the conditions \textsf{R0} and \textsf{R3}; however, the parthood $\subseteq$ can ensure additional properties. In this proof, it is assumed that $\bot = \emptyset$ and can be replaced by some other nonempty set. 
\begin{enumerate}
\item {\begin{enumerate}
 \item {\textsf{lA-cmo} holds because if a granule $h\subseteq a^{l_\alpha}$,}
 \item {then $h\subseteq b \subseteq a$ and $\kap (a,h) \geq 1- \alpha$.}
 \item {Thus by \textsf{RV}, this implies $\kap (b,h) \geq 1- \alpha$. Therefore, $h\subseteq b^{l_\alpha}$.}
\end{enumerate}}
\item {\textsf{uA-mo} holds because 
 \begin{enumerate}
 \item {if a granule $h\subseteq a^{u_\alpha}$, then $h\cap a \neq \emptyset$.}
 \item {This yields $h \cap b \neq \emptyset$ and $\kap (a,h) > \alpha$.}
 \item {By \textsf{RV}, this implies $\kap (b,h) > \alpha$. Therefore, $h\subseteq b^{u_\alpha}$. }
 \end{enumerate}}
\item { \begin{enumerate}
 \item {If a granule $h\subseteq a^{l_\alpha}$ and $h\subseteq b^{l_\alpha}$,}
 \item {then $h\subseteq a$, $h\subseteq b$ and $h\subseteq a\cap b$.}
 \item {Moreover, $\kap(a, h)\geq 1 - \alpha$ and $\kap(b, h)\geq 1 - \alpha$. }
 \item {Therefore, by \textsf{RV}, $\kap(a\cap b, h) \geq 1-\alpha$. }
 \item {This ensures $h\subseteq (a\cap b)^{l_\alpha}$, and proves $a^{l_\alpha} \cap b^{l_\alpha} \subseteq (a\cap b)^{l_\alpha}$.}
 \end{enumerate}}
\item {If a granule $h \subseteq a^{l_\alpha}$, 
 \begin{enumerate}
 \item {then $h\subseteq a^{l_\alpha}\subseteq a$ and $\kap(a, h)\geq 1-\alpha $.}
 \item {By \textsf{RV}, this yields $\kap(a^{l_\alpha}, h)\geq 1-\alpha$.}
 \item {From this it follows that $h \subseteq a^{l_\alpha l_\alpha}$. Thus $ a^{l_\alpha}\subseteq a^{l_\alpha l_\alpha}$ . Because $a^{l_\alpha l_\alpha} \subseteq a^{l_\alpha}$ (from the definition)  equality holds and proves \textsf{lA-idem}.}
 \end{enumerate}}
\item {If a granule $h \subseteq a^{l_\alpha}$, 
 \begin{enumerate}
 \item {then $h\subseteq a^{l_\alpha}\subseteq a$ and $\kap(a, h)\geq 1-\alpha $.}
 \item {However if $\kap(a, h)\geq 1-\alpha $ then $\kap(a, h) > \alpha $ (as $\alpha \in (0, 0.5)$)}
 \item {Thus $h\subseteq a^{u_\alpha}$, and \textsf{luA} holds.}
 \end{enumerate}}
\end{enumerate}
\end{proof}

\begin{remark}
The condition \textsf{R3} is not used in the proofs! Therefore the next theorem is proved.
\end{remark}

\begin{theorem}\label{prifalp}
In the statement of Theorem \ref{alpha}, \textsf{wqRIF} can be replaced by \textsf{pRIF}.
\end{theorem}

If instead pointwise $*$-lower and $*$-upper approximations are defined as follows (by dropping the condition $h\subseteq B $):
\begin{align*}
B^{l_{\alpha^*p}} = \{x:\, n(x)\in \mathcal{F} \, \&\,\kap (B, n(x)) \geq 1-\alpha \}   \tag{$l^*$-vprs}\\
B^{u_{\alpha^*p}} =  \{x:\, n(x)\in \mathcal{F}\,\&\,\kap (B, n(x)) > \alpha \}   \tag{$u^*$-vprs}
\end{align*}
then they would additionally satisfy properties that are not typical of approximations. For specific $\kap$, these are known in the literature \cite{kz,qiao18}. Lower approximations would not be subsets of the approximated set and upper approximations may not include them. Still it would be the case that the lower approximation is included in the upper. This type of approximations would be referred to as \emph{peripatetic}. 

\begin{definition}
The $l^*$ and $u^*$-vprs are pointwise approximations, and the divergent granular version of peripatetic VPRS approximations would be as below
\begin{align*}
B^{l_\alpha^*} = \bigcup \{h:\, h\in \mathcal{G}\, \&\,\kap (B,h) \geq 1-\alpha \}   \tag{l-vprs*}\\
B^{u_\alpha^*} = \bigcup \{h:\, h\in \mathcal{G}\, \&\,\kap (B,h) > \alpha \}   \tag{u-vprs*}
\end{align*} 
\end{definition}

\begin{remark}
Hybrid variants of rough approximations have properties similar to peripatetic approximations. Furthermore, note that the approximation is constructed relative to the set being approximated. In the absence of parthood properties between the granule and a set, it does make sense to consider the case using $\kap (\text{granule, set})$ too when the granule has more reliable means of interpretation by itself or through an additional mechanism/process. Air quality indices for different regions in cities are approximations. Related data can, for example, be processed through both approaches.
\end{remark}

The properties satisfied by the $\alpha^*$ approximations are proved in the next theorem.

\begin{theorem}\label{ast}
Over a rough convenience $\mathbb{S}$, if $\kap$ satisfies \textsf{R0} and \textsf{RV}, and $l_{\alpha^*}$ and $u_{\alpha^*}$ are as above, then all of the following hold:
\begin{align}
(\forall a, b)(a^{l_{\alpha^*}}\subseteq b\subseteq a \longrightarrow a^{l_{\alpha^*}}\subseteq b^{l_{\alpha^*}})   \tag{lA-cmo*}\\
(\forall a, b)(a\subseteq b\subseteq a^{u_{\alpha^*}}  \longrightarrow a^{u_{\alpha^*}}\subseteq b^{u_{\alpha^*}})   \tag{uA-cmo*}\\
(\forall a) a^{l_{\alpha^*}} \subseteq a^{u_{\alpha^*}} \tag{luA*}\\
(\forall a) a^{l_alpha} \subseteq a^{l_{\alpha^*}} \,\&\, a^{u_alpha} \subseteq a^{u_{\alpha^*}} \tag{luAA}
\end{align}
\end{theorem}
\begin{proof}
\begin{enumerate}
\item {\begin{enumerate}
 \item {\textsf{lA-cmo*} holds because if a granule $h\subseteq a^{l_{\alpha^*}}$,}
 \item {then $h\subseteq b \subseteq a$ and $\kap (a,h) \geq 1- \alpha$.}
 \item {Thus by \textsf{RV}, this implies $\kap (b,h) \geq 1- \alpha$. Therefore, $h\subseteq b^{l_{\alpha^*}}$. The argument works because approximations are unions of granules. }
\end{enumerate}}
\item {\textsf{uA-cmo*} holds because 
 \begin{enumerate}
 \item {if a granule $h\subseteq a^{u_{\alpha^*}}$, then $h\cap a \neq \emptyset$.}
 \item {This yields $h \cap b \neq \emptyset$ and $\kap (a,h) > \alpha$.}
 \item {By \textsf{RV}, this implies $\kap (b,h) > \alpha$. Therefore, $h\subseteq b^{u_{\alpha^*}}$. }
 \end{enumerate}}
\item {If a granule $h \subseteq a^{l_{\alpha^*}}$, 
 \begin{enumerate}
 \item {then $\kap(a, h)\geq 1-\alpha $.}
 \item {However this implies $\kap(a, h) > \alpha $ (as $\alpha \in (0, 0.5)$)}
 \item {Thus $h\subseteq a^{u_{\alpha^*}}$. This proves \textsf{luA*}.}
 \end{enumerate}}
\end{enumerate}
The last property \textsf{luAA} easily follows from the definition.
\end{proof}

The property \textsf{RI} can be sufficient to prove the following:

\begin{theorem}\label{riprop}
In a rough convenience, if $\kap$ satisfies \textsf{R0} and \textsf{RI} then 
\[a^{l_{\alpha}} \cap b^{l_{\alpha}} \subseteq (a\cap b)^{l_{\alpha}} \tag{lARI-cap}\] 
\end{theorem}
\begin{proof}
\begin{enumerate}
\item {If a granule $h\subseteq a^{l_{\alpha}}$ and $h\subseteq b^{l_{\alpha}}$, then $h\subseteq a$, $h\subseteq b$,  $\kap (a, h)\geq 1- \alpha$ and $\kap (b, h)\geq 1- \alpha$.}
\item {By \textsf{RI}, this yields $\kap (a\cap b, h)\geq 1- \alpha$.}
\item {From all this, it follows that $h\subseteq (a\cap b)^{l_{\alpha}}$. Therefore, $a^{l_{\alpha}} \cap b^{l_{\alpha}} \subseteq (a\cap b)^{l_{\alpha}}$ and \textsf{lARI-cap} holds. }
\end{enumerate}
\end{proof}

\subsection{Observations}

In the VPRS approach, all of the following apply. 
\begin{itemize}
\item {The idea of granules and definite objects is often conflated,}
\item {The idea of lower approximation is transformed to as \emph{one that includes the definites that are sufficiently included in the object being approximated}. This deviates severely from the modal meaning of being \emph{certainly included}.}
\item {The semantic domain (or meta-level for simplicity) in which the approximations are constructed can be read as one that involves approximate inclusions (and therefore a specific kind of parthood). There is scope for interpreting things relative to other parthood predicates too.}
\item {A granule is valid for an approximation if it is sufficiently large.}
\item {The size of a granule is a measure of its participation in the approximation processes.}
\end{itemize}

In \cite{kz}, a $\beta$-u-granular approximation is studied with $\alpha$-l-granular essentially with the assumption that $\beta \leq \alpha$ (in the paper $1-\kap (a, b)$ is used instead of the rough inclusion function). This is an asymmetric variant of VPRS. In the paper, the complement of the classical RIF is used instead of a generalized variant thereof. While hybrid variants are studied, the number of papers on modifications of the core assumptions of  

A number of generalizations that stick to the numeric perspective of parthood associated that make use of generalized RIFs \cite{am111111} defined in section \ref{grif} are possible. This permits a higher order perspective of looking at associated approximations in a parametrized way. 

For the scenario, it makes sense to regard approximations as parametrized by the generalized RIF $\kap$, the type (formula or index) $\tau$, and the precision parameter $\alpha$. Therefore a general VPRS lower (gvl) approximation defined under the generalized VPRS schematics can be written in the form $l(\kappa, \tau, \alpha)$. Similarly, general VPRS upper (gvu) approximations will be written in the form  
$u(\kappa, \tau, \alpha)$. Let the collection of all such approximations on $S$ be $L(\kappa, \tau, \alpha)(S)$ and $U(\kappa, \tau, \alpha)(S)$ respectively. The properties of these will be considered in a separate paper; however, it helps to read the defined approximations in this perspective because a number of hybrid variants using VPRS can be potentially reformulated in these terms.

\subsection{Substantial Parthood}\label{subpart}

Ideas of substantial parthood need to satisfy a number of conditions, and abstractions are actually possible in suitable domains of discourse. In the context of granular VPRS, predicates possibly interpretible as forms of substantial parthood can be defined as follows (the labeling is made consistent with those used for granular graded rough sets by the first author):

For any $a, b\in S$, a fixed positive integer $k$ and a fixed $\alpha$
\begin{align*}
\pc_s^5 {ab} \text{ iff } a^{l_\alpha}\subseteq b^{l_\alpha}   \tag{s5}\\
\pc_s^6 {ab} \text{ iff } a\subseteq b \& \#(a)>k   \tag{s6}\\
\pc_s^7 {ab} \text{ iff } (\forall h\in \mathcal{G})(h\subseteq a,\, \&\, h\subseteq b, \&\,\kap (a, h) \geq 1- \alpha \longrightarrow \kap (b, h) \geq 1- \alpha )   \tag{s7}\\
\pc_s^* {ab} \text{ iff } \#(a\cap b) >k \& b\not\subset a  \tag{s*}\\
\pc_s^9 {ab} \text{ iff } (\forall h\in \mathcal{G})(\kap (a, h) \geq \alpha \longrightarrow \kap (b, h) \geq \alpha )   \tag{s9}\\
\pc_s^3 {ab} \text{ iff } \#(a\cap b) >k \& a\subseteq b  \tag{s3}\\
\pc_s^l ab \text{ iff } a^{l_\alpha}\subseteq b^{l_\alpha} \,\&\, \kap(a, b) \geq 1-\alpha  \tag{s0l}\\
\pc_s^u ab \text{ iff } a^{u_\alpha}\subseteq b^{u_\alpha} \,\&\, \kap(a, b) \geq \alpha  \tag{s0u}\\
\pc_s^t ab \text{ iff } \exists h\in \mathcal{T}\subseteq \mathcal{G} h\subseteq a \subseteq b    \tag{st} 
\end{align*}

Since $\alpha\in [0, 0.5)$, the constant $1-\alpha$ is used to indicate a specific range. The conditions \textsf{s5, s7, s9, s0l, and s0u} are dependent on granules, definite objects and the function $\kap$. Thus if $\kap$ is granular, then all are granular. Otherwise, only \textsf{s5} is.
Further for conditions involving $l_\alpha$ or $u_\alpha$, a star ($*$) version can be defined by replacing them by $l_{\alpha^*}$ and $u_{\alpha^*}$ respectively. In the last definition \textsf{st}, $\mathcal{T}$ may be read as a set of special granules.

The basic properties of these relations are stated in the next four propositions.

\begin{proposition}
$\pc_s^5$ satisfies the following:
\begin{enumerate}
\item {$(\forall a)\,  \pc_s^5 aa $}
\item {$(\forall a, b, c)\, (\pc_s^5 ab \& \pc_s^5 bc \longrightarrow \pc_s^5 ac) $ }
\item {$\pc_s^5 $ is not antisymmetric or symmetric in general.}
\item {$\pc_s^5 ab $ if and only if $\pc_s^7 ab$ }
\end{enumerate}
\end{proposition}
\begin{proof}
\begin{enumerate}
\item {For any $a\in S$, $a^{l_\alpha} = a^{l_\alpha}$. Thus $\pc_s^5 aa $ holds.}
\item {Transitivity holds because $a^{l_\alpha} \subseteq b^{l_\alpha}\subseteq c^{l_\alpha}$ implies $a^{l_\alpha} \subseteq c^{l_\alpha}$.}
\item {$\pc_s^5 ab $ if and only if $\pc_s^7 ab$ follows from the basic properties of $l_\alpha$}
\end{enumerate}
\end{proof}

\begin{proposition}
$\pc_s^6$ satisfies the following:
\begin{enumerate}
\item {$\pc_s^6 aa $ if and only if $\#(a) > k$}
\item {$(\forall a, b, c)\, (\pc_s^6 ab \& \pc_s^6 bc \longrightarrow \pc_s^6 ac) $ }
\item {$\pc_s^6 $ is antisymmetric; however, it is not symmetric in general.}
\end{enumerate}
\end{proposition}
\begin{proof}
The proof is by direct verification and obvious counterexamples.
\end{proof}

$\pc_s^l ab$ basically means that $a$ is close enough to $b$ for the defining conditions to happen. If $a\subset b $, then it can happen that $\neg \pc_s^l ab$.

\begin{proposition}
$\pc_s^l$ satisfies all of the following properties for a fixed wqRIF $\kap$ and $\alpha$:
\begin{enumerate}
\item {$(\forall a)\, \pc_s^l aa$}
\item {Transitivity, antisymmetry and symmetry fail in general.}
\end{enumerate}
\end{proposition}

\begin{proposition}
$\pc_s^u ab$ satisfies all of the following 
\begin{enumerate}
\item {$ \pc_s^u aa$}
\item {$(\pc_s^u ab \, \&\, \pc_s^u ba \longrightarrow a^{u_\alpha} = b^{u_\alpha}) $}
\item {For every $b$, $\pc_s^u \bot b$}
\end{enumerate}
\end{proposition}

\begin{proof}
\begin{enumerate}
\item {Clearly if $a=b$, then $a^{u_\alpha}\subseteq a^{u_\alpha} \,\&\, \kap(a, a) =1 \geq \alpha $. Thus $\pc_s^u aa$ follows.}
\item {$(\pc_s^u ab \, \&\, \pc_s^u ba \longrightarrow a^{u_\alpha} = b^{u_\alpha}) $ follows from the definition.}
\item {For any $b$, $\kap (\bot, b) = 1 \geq \alpha$. Further as $\bot ^{u_\alpha} \subseteq b^{u_\alpha}$, it follows that $\pc_s^u \bot b$.}
\end{enumerate}
\end{proof}

\begin{definition}
If $\kap$ is defined on a pGSGS $S$, then for any $a\in S$, let \[E_1(a,b) =\{c:\, \kap(a, b) = \kap(a, c)\} \text{ and} \]
\[E_2(a,b) =\{c:\, \kap(a, b) = \kap(c, b)\}\] 
$E_1$ and $E_2$ will be referred to as the \emph{1-equalizer} and \emph{2-equalizer} function respectively.
\end{definition}

The equalizer functions are important because they help in characterizing the roughly equal sets (or rough objects).

\begin{proposition}
$\pc_s^*$ satisfies the following:
\begin{enumerate}
\item {$\pc_s^* aa$ if and only if $\#(a)>k$,}
\item {$\pc_s^*$ is not transitive in general, }
\item {$\pc_s^* $ is neither antisymmetric nor symmetric in general.}
\end{enumerate}
\end{proposition}
\begin{proof}
\begin{enumerate}
\item {For any $a, b$, $\pc_s^* {ab} \text{ if and only if } \#(a\cap b) >k \& b\not\subset a$ by definition. Thus if $\#(a) >k$, $a=a$ and $\pc_s^* {aa}$ follows.}
\item {Let $k=4$, $a= \{1, 2, 3, 4, 5, 6, 7, 8, 9\}$,\\ $b= \{20, 15, 1, 2, 3, 4, 5 \}$, and $c=\{20, 12, 1, 2, 3, 30$. Then it can be checked that $\pc_s^* ab$, and $\pc_s^* bc$ hold; however, $\pc_s^* ac$ does not hold. Therefore, transitivity fails in general. }
\item {In the context of the same example, if $h = \{1, 2, 3, 4, 5\}$, $\pc_s^* ah$ holds; however, $\pc_s^* ha$ fails. The failure of antisymmetry can be checked with an easy counterexample.}
\end{enumerate}
\end{proof}

\begin{proposition}
$\pc_s^3$ satisfies the following:
\begin{enumerate}
\item {$\pc_s^3 aa$ holds if and only if $\#(a)>k$,}
\item {$\pc_s^3$ is transitive,}
\item {$\pc_s^3 $ is antisymmetric and not symmetric in general.}
\end{enumerate} 
\end{proposition}

\begin{proof}
The proof is by direct verification and obvious counterexamples.
\end{proof}

\begin{proposition}
$\pc_s^t$ is a transitive, antisymmetric relation. 
\end{proposition}

The ideas of substantial parthood predicates defined above will be compared with axiomatic ideas of \emph{rational approximation} in the generalized VPRS contexts within the framework of section \ref{fra}.

\section{Generalized VPRS and Graded Rough Sets}

Connections of generalized granular VPRS with any hybrid rough set theory that relies on numeric valuation functions for determining approximations are of natural interest for both theoretical reasons and applications. The Bayesian approach to rough sets in \cite{sdz05} can additionally be interpreted as a generalization of VPRS in which the misclassification parameters are replaced by prior probabilities. If approximations are constructed with derived operations on granules, then granular generalized VPRS can be expected to be related to graded rough sets. However the connection and translation of results are not necessarily straightforward especially when they are formulated over a \textsf{pGsGS} or even a set. This last aspect is explored in this section and new results are proved.

A \emph{grade} $k$ can be any fixed positive integer, and is to be related to the cardinality of granules or sets used in the context. Let $S$ be a collection of sets (that are subsets of a $H$), and $\mathcal{G}$ a subset of $S$.  The basic part of relation will be taken to set inclusion. If $k$ is a fixed positive integer and $x\in S$, then the \emph{$k$-lower} and \emph{$k$-upper} approximations and related regions will be ($\#$ being the cardinality function $: S \longmapsto N$)
\begin{align}
x^{u_k} = \bigcup \{h:\, h\in \mathcal{G}\, \&\, \#(h\cap x) > k\}   \tag{k-upper}\\
x^{l_k} = \bigcup \{h:\, h\in \mathcal{G}\, \&\,\#(h) - \#(h\cap x) \leq k\}   \tag{k-lower}\\
Pos_k(x) = x^{u_k}\cap x^{l_k}   \tag{k-positive region}\\
Neg_k(x) = H\setminus (x^{l_k}\cup x^{u_k})   \tag{k-negative region}\\
Bnd_k^u(x) = x^{u_k}\setminus x^{l_k}    \tag{upper k-boundary}\\
Bnd_k^l(x) = x^{l_k}\setminus x^{u_k}    \tag{lower k-boundary}
\end{align}

In the same context, for a given $\alpha\in [0, 0.5)$ and a \textsf{wqRIF} $\kap$, it may be conditionally possible to define the generalized VPRS approximations $l_\alpha, \, u_\alpha , \, {u_\alpha^*}$ and ${l_\alpha^*}$ with no additional assumptions on $\mathcal{G}$ (that may be interpreted as a granulation under the additional assumptions of section \ref{fra}). In the literature, neighborhood based pointwise approximations have been studied \cite{yl96}; however, these are even less related to granular VPRS approximations.

If $\kap$ is taken as $\nu$ (or defined by similar ratios of cardinalities), then 

\begin{theorem}\label{vpgrau}
For a fixed $\alpha\in (0, 0.5)$, and $\kap = \nu$ there exists a set of integers $K$ such that $S$ is partitioned into $\{S_i\}_{i\in K}$ and \[(\forall x\in S_i)\, x^{u_\alpha^*} = x^{u_i}\] 
\end{theorem}

\begin{proof}
\begin{enumerate}
\item {Suppose that $x^{u_\alpha^*} = \bigcup \{h:\, h\in \mathcal{G}\, \&\,\nu (x,h) > \alpha \}$ for some $\alpha\in (0, 0.5)$, then $\dfrac{\#(h\cap x)}{\#(x)} \geq \alpha $. }
\item {Let $\#(x) = k_x$.}
\item {Because $\#(h\cap x)$ is an integer, the condition in the first line can be written as $\#(h\cap x) \geq \ulcorner(\alpha k_x)\urcorner$ (the ceiling function being denoted by $\ulcorner \urcorner$.}
\item {Over the collection of all such integers of the form $\ulcorner(\alpha k_x)\urcorner$, neither the least, nor the greatest integer need to work uniformly such that $u_k = u_\alpha^*$ in all cases.}
\item {This shows that a single generalized VPRS upper approximation operator can be defined by a number of graded upper approximation operators, corresponding to a set of integers $K$ under \[(\forall x\in S_i) x^{u_\alpha^*} = x^{l_i}\] with $\{S_i\}_{i\in K}$ being a partition of $S$.}
\end{enumerate}
\end{proof}

\begin{remark}
For a fixed integer $k$, and $\kap = \nu$, the granular $k$-grade upper approximation $u_k$ need not be representable by a finite number granular generalized VPRS approximations. This can be seen in the following proof form: 
\begin{enumerate}
\item {Suppose $\#(h\cap x) \geq k$, $\#(x) = r_x\leq \#(H) = r$ for some positive integers $r_x$ and $r$.}
\item {Thus $\dfrac{\#(h\cap x)}{\#(x)} \geq \frac{k}{r_x}$}
\item {Let $\frac{k}{r_x} = \alpha_x$. Identify those $x$ with identical values of $\alpha_x$ }
\item {Now partition $S$ into a finite number of subsets $\{S_i\}_{i\in F}$ for which $\alpha_x$ is the same. }
\item {All this does not ensure that $\alpha_x \in (0,0.5)$. This proves that if $x\in S_i$ then $x^{u_{\alpha_x}^*} = \bigcup \{h:\, h\in \mathcal{G}\, \&\,\nu (x,h) > \alpha_x \}$ may be the same as $x^{u_k}$}
\end{enumerate}
\end{remark}

\begin{theorem}\label{vpgral}
For a fixed $\alpha\in (0, 0.5)$, and $\kap = \nu$ there exists a set of integers $K$ such that $S$ is partitioned into $\{S_i\}_{i\in K}$ and \[(\forall x\in S_i)\, x^{l_\alpha^*} = x^{l_i}\] 
\end{theorem}
\begin{proof}

\begin{enumerate}
\item {$x^{l_\alpha^*} = \bigcup \{h:\, h\in \mathcal{G}\, \&\,\nu (x,h) \geq 1-\alpha \}$.}
\item {The defining condition can be simplified into $\#(x\cap h) \geq (1-\alpha) \#(x) = (1-\alpha) k_x$ (say)}
\item {Because $\#(h\cap x)$ is an integer, the condition in the first line can be written as $\#(h\cap x) \geq \ulcorner(1-\alpha) k_x\urcorner$ (the ceiling function being denoted by $\ulcorner \urcorner$).}
\item {Let $K$ be the set of all integers of the form $k_x$.}
\item {This shows that a single generalized VPRS lower approximation operator can be defined by a number of graded upper approximation operators, corresponding to a set of integers $K$ under \[(\forall x\in S_i) x^{l_\alpha^*} = x^{l_i}\] with $\{S_i\}_{i\in K}$ being a partition of $S$}
\item {Obviously for different partitions, a distinct $k$ needs to be used.}
\end{enumerate} 
\end{proof}

For the approximations $l_\alpha$ and $u_\alpha$, additional constraints need to be imposed on the graded approximations, and both theorems \ref{vpgrau} and \ref{vpgral} carry over. 
 
\textsf{Thus in essence when $\kap$ is a ratio of cardinalities, some correspondences between generalized granular VPRS approximations and granular graded approximations are possible. However granular graded approximations may not be representable in terms of a number of granular VPRS approximations}.

\section{Framework for Rational Approximations}\label{fra}

Because attributes and granules have meanings associated, representation of rationality of approximations requires additional predicates or valuations in most set-theoretic situations. Those based on cardinalities of sets or probabilistic measures are typically applicable when the attributes have equal or weighted contribution to meaning in an arithmetic sense. To improve upon this frameworks are essential for both practical applications, and theoretical representation. In the earlier study on rational approximations in general graded rough sets by the first author, it was shown that multiple concepts (though closely related) of substantial parthood are relevant for the discourse.  

In the context of general rough sets, it is possible to distinguish between an approximation operator $f$ being rational and rationality being a property of concrete approximations (like that of $x^f$ relative to something else). VPRS-style definitions try to ensure a form of the latter (and that may work under very nice uncommon conditions). Graded rough sets and the granular variants additionally try to define a rationality of the latter type based on the size of granules. However the former is certainly not independent of the latter. 

Possible candidates of the additional predicates mentioned above are the closely related \emph{is an essential part of} $\pc _e$, \emph{is an inessential part of} $\pc _n$, \emph{is a relevant part of} $\pc _r$, \emph{is a substantial defining part of} $\pc_s$ and variations thereof. As these are not always definable in an easy way in a GGSP relative to observations in a real-life situation (especially when numeric valuations are difficult), it is preferable to add one of them to the model. In this research $\pc_s$ will be preferred because it does not commit to a specific ontology. The predicate is very closely related to being \emph{a substantial part of} (that is somewhat known in the mereology literature \cite{ham2017}).

\begin{definition}\label{subst}
A \emph{Pre-General Pre-Substantial Granular Space} (\textsf{pGpsGS}) $\mathbb{S}$ will be a partial algebraic system  of the form \[\mathbb{S} \, =\, \left\langle \underline{\mathbb{S}}, \gamma, l , u, \pc, \pc_s , \leq , \vee,  \wedge, \bot, \top \right\rangle\] with $\left\langle \underline{\mathbb{S}}, \gamma, l , u, \pc, \leq , \vee,  \wedge, \bot, \top \right\rangle$ being a \textsf{Pre*-GGS} and satisfies \textsf{sub3-sub4} ($\approx $ is intended as any definable rough equality):
\begin{align}
(\forall a, b)\, (\pc_s ab \& \pc_s ba \longrightarrow a\approx b)   \tag{sub3}\\
(\forall a, b, e)\, (\pc_s ae \& \pc_s ab \longrightarrow  \pc_s a(b\vee e) )    \tag{sub4}
\end{align}
If in addition, \textsf{sub1, sub2, sub5}, \textsf{UL1} and \textsf{sub6} are satisfied, then $\mathbb{S}$ will be said to be a \emph{Pre-General Substantial Granular Space} (\textsf{pGsGS}).
\begin{align}
(\forall a)\, \pc_s aa    \tag{sub1}\\
(\forall a, b)\, (\pc_s ab \longrightarrow \pc ab)     \tag{sub2}\\
(\forall a, b, e)\, (\pc_s ba \& \pc_s be \& \pc ae \longrightarrow  \pc_s ae)   \tag{sub5}\\
(\forall a, b, e)\, (\pc_s ab \& \pc_s eb \& \pc ae \longrightarrow  \pc_s ae)   \tag{sub6}
\end{align}
If a pGsGS satisfies the condition \textsf{PT1}, then it will be referred to as a \emph{General Substantial Granular Space}.
\end{definition}

The defining conditions \textsf{sub1-sub2} say that anything is a d-substantial part of itself, if something is a d-substantial part of something else, then it is additionally a part of that something else. The condition \textsf{sub3} ensures that if something is a d-substantial part of something else and conversely, then they have something in common. \textsf{sub4} says that if one thing is a substantial part of two other things, then it is a substantial part of their generalized join (possibly a generalized s-norm when $\pc$ and $\pc_s$ are transitive and antisymmetric). It may be noted that a special case is when $\pc$ and $\pc_s$ are partial orders and $\vee$ is a generalized s-norm (see \cite{zhang2005}). It might appear that \textsf{sub1, sub5} and \textsf{sub6} are too strong in some cases, and that is the motivation for introducing at least two levels of the concept. 

A condition that can fail to hold in many contexts because of conflict with parthood is: 
\[(\forall a, b, e)\, (\pc_s ae \& \pc_s be \longrightarrow  \pc_s (a\vee b) e) \tag{Asy}\]

Note that a binary relation $R$ is said to be \emph{r-euclidean} if and only if the condition
$(\forall a, b, c) (Rab \& Rac \longrightarrow Rbc)$. $R$ is \emph{l-euclidean} if and only if $(\forall a, b, c) (Rba \& Rca \longrightarrow Rbc)$. The conditions \textsf{sub5} and \textsf{sub6} are respectively generalizations of these through two predicates.

\begin{definition}\label{rlow}
In a \textsf{pGpsGS} $\mathbb{S}$, \emph{rational lower approximation} ($l_s$) of a $a\in \ms$ will be a lower approximation of some $b$ that satisfies
\begin{itemize}
\item {$\pc ba$}
\item {$(\forall e)(e^l= e\,\&\, \pc eb \longrightarrow \pc_s ea)$}
\end{itemize}
\end{definition}

\begin{definition}\label{rup}
In the context of the above definition, a \emph{rational upper approximation} ($u_s$) of $a\in \ms$ will be some $b$ that satisfies   
\begin{itemize}
\item {$(\exists z) z^u = b$}
\item {$\pc ab \, \&\, \pc ba^u$}
\item {$(\forall x)(x^l =x\,\&\, \pc xb \longrightarrow \pc_s xb )$}
\end{itemize}
\end{definition}

Clearly there is a big difference between upper and rational upper approximations. It can happen that rational upper approximations do not exist even in the context of classical rough sets (when it is enhanced with restrictions on \emph{substantial inclusion}). Therefore it is necessarily a partial operation on a \textsf{pGpsGS}. By contrast, rational lower approximations exist always.

\begin{proposition}
In the context of Definition \ref{rup}, the following hold in a \textsf{pGsGS}:
\begin{align*}
(\forall x) x^{l_s l_s} = x^{l_s}   \tag{Idempotence}\\
(\forall x) \pc x^{l_s} x^l    \tag{Low-comp1}\\
(\forall a, b) (\pc ab \longrightarrow \pc a^{l_s}b^{l_s} )  \tag{s-Monotony}\\
(\forall x) \pc x^{u_s} x^u   \tag{Up-comp1}
\end{align*}
\end{proposition}

\begin{proof}
\begin{enumerate}
\item {Idempotence is inherited from $l$ in the context of the definition.}
\item {This again follows from the definition }
\item {s-Monotony is induced from that of $l$ again}
\end{enumerate}
\end{proof}

An important compatibility problem in a \textsf{pGpsGS} is about finding conditions that ensure the following:
\begin{align*}
(\forall x) \pc_s x^{l_s} x^l   \tag{Low-comp2}\\
(\forall x) \pc_s x^{u_s} x^u   \tag{Up-comp2} 
\end{align*}

\subsection{Philosophy of Rational Intents}

This subsection is intended to contextualize the approach adopted in this research.

\emph{An approximation operator has rational intent if that which is being approximated has much in common with those that are involved in the approximation process (as specified by the operator)}.
This idea involves multiple semantic domains as listed below:
\begin{itemize}
\item {An approximation operator has rational intent: this can be realized in the classical domain or a domain that is a part of the classical domain}
\item {The object being approximated: can be assumed to be in the classical domain by default.}
\item {The approximation can be in distinct rough domains $\mathfrak{R_1}\ldots \mathfrak{R_n}$ because a single operator may be associated with multiple domains. For simplicity or a relatively reductionist approach all these can be taken as one.}
\end{itemize}

A number of perspectives for defining ideas of \emph{rational intent} are possible within the contexts of general rough sets. While the best may require external frameworks for their validation, it is always better to internalize most aspects while modeling.  
\begin{description}
\item [Ri1]{An approximation operator has rational intent if its range of errors is small.}
\item [Ri2]{An approximation operator has rational intent if its errors are uniformly distributed.}
\item [Ri3]{An approximation operator has rational intent if it is consistent.}
\item [Ri4]{An approximation operator has rational intent if its realization is coherent with an ideal rough context.}
\item [Ri5]{An approximation operator has rational intent if its realization is part of an ideal rough context.}
\item [Ri6]{An approximation operator has rational intent if it is a substantial part of the thing being approximated or the converse is true.}
\end{description}
While Ri4 and Ri5 explicitly refer to an external context, Ri1-Ri3 may additionally involve relatively external context for a definition of errors and consistency. In principle, if all aspects of the context are properly done, then Ri6 can be the best approach.

\section{Examples}

Part of this example was originally constructed for bited approximations in \cite{am105} by the first author and additional details considered by her in \cite{am501}. Let $S= \left\langle \underline{S}, R \right\rangle$ be a general approximation space with $R$ being a binary relation on the set $\underline{S}$. For a granulation $\mathcal{G}\subset \wp(S)$ (understood as a proper cover), the lower, upper and bited upper approximations of a $ a\subseteq S $ are defined as follows:

\begin{align*}
a^l \, =\,\bigcup\{x\,:\,x \subseteq a,\,a \in \mathcal{G} \} \tag{Lower}\\
a^u \, =\,\bigcup\{x\,:\,a\,\cap\,x\,\neq\,\emptyset,\,x \in \mathcal{G} \} \tag{Upper}\\
POS_{\mathcal{G}}(a) \, =\, a^l \tag{Positive Region}\\
NEG_{\mathcal{G}}(a) \, =\,a^{cl} \tag{Negative Region}\\
a^{u_b} \, =\,a^u\,\setminus\,a^{cl} \tag{Bited Upper}
\end{align*}

Let $H \, =\,\{x_{1}, x_{2}, x_{3}, x_{4}\}$ and $T$ be a tolerance $T$ on it generated by 
\[\{(x_{1}, x_{2}),\,(x_{2},x_{3})\}.\]

Denoting the statement that the granule generated by $x_{1}$ is $(x_{1},\,x_{2})$ by $(x_{1}:x_{2})$, let the granules be the set of predecessor neighborhoods:  \[\mathcal{G}=\{(x_{1}:x_{2}),\,(x_{2}:x_{1},x_{3}),\,(x_{3}:x_{2}),\,(x_{4}:)\}\]

The different approximations (lower ($l$), upper ($u$) and bited upper ($u_b$)) are then as in Table.\ref{bitedap} below. For $\alpha = 0.3$, the granular VPRS approximations ($l_\alpha$, and $u_\alpha$) are additionally included in the table. It is of interest to look at possible substantial part of relations that make the VPRS, standard, and bited approximations rational or are otherwise interesting. The braces on sets have been omitted. For $k=1$, the $1$-graded approximations and the $\alpha^*$ VPRS approximations (for $\alpha= 0.3$) are in table \ref{gr1}.

\begin{table}[hbt]
\begin{tabular}{|c|c|c|c|c|c|c|}
\hline\hline
\textbf{Set} & $\mathbf{a}$ & $\mathbf{a^l}$ & $\mathbf{a^u}$  & $\mathbf{a^{u_b}}$ & $\mathbf{a^{l_\alpha}}$ &  $\mathbf{a^{u_\alpha}}$  \\
\hline 
$A_{1}$ & $x_{1}$ & $\emptyset$ & $x_{2},\,x_{1}$  & $x_{1}$ & $\emptyset$ &  $x_1, x_2, x_3$   \\
\hline
$A_{2}$ & $x_{2}$ & $\emptyset$ & $x_{1}, x_{2}, x_{3}$  & $x_{1}, x_{2}, x_{3}$ & $\emptyset$ &   $x_1, x_2, x_3$\\
\hline
$A_{3}$ & $x_{3}$ & $\emptyset$ & $x_{1}, x_{2}, x_{3}$  & $x_{3}$  & $\emptyset$ & $x_1, x_2, x_3$\\
\hline
$A_{4}$ & $x_{4}$ & $x_{4}$ & $x_{4}$ & $x_{4}$ & $x_{4}$ & $x_{4}$ \\
\hline
$A_{5}$ & $x_{1}, x_{2}$ & $x_{1}, x_{2}$ & $x_{1}, x_{2}, x_{3}$  & $x_{1}, x_{2}, x_{3}$ & $x_1, x_2$ &  $x_1, x_2, x_3$  \\
\hline
$A_{6}$ & $x_{1}, x_{3}$ & $\emptyset$ & $x_{1}, x_{2}, x_{3}$  & $x_{1}, x_{2}, x_{3}$ & $\emptyset$ &  $x_1, x_2, x_3$  \\
\hline
$A_{7}$ & $x_{1}, x_{4}$ & $x_{4}$ & $H$  & $x_{1}, x_{4}$ & $\emptyset$  & $H$ \\
\hline
$A_{8}$ & $x_{2}, x_{3}$ & $x_{2}, x_{3}$ & $x_{1}, x_{2}, x_{3}$  & $x_{1}, x_{2}, x_{3}$ & $x_2, x_3$ &  $x_1, x_2, x_3$ \\
\hline
$A_{9}$ & $x_{2}, x_{4}$ & $x_{4}$ & $H$  & $H$ & $\emptyset$ &  $H$ \\
\hline
$A_{10}$ & $x_{3}, x_{4}$ & $x_{4}$ & $H$  & $x_{3}, x_{4}$ & $\emptyset$ & $H$ \\
\hline
$A_{11}$ & $x_{1}, x_{2}, x_{3}$ & $x_{1}, x_{2}, x_{3}$ & $x_{1}, x_{2}, x_{3}$  & $x_{1}, x_{2}, x_{3}$ & $x_1, x_2, x_3$ &   $x_1, x_2, x_3$  \\
\hline
$A_{12}$ & $x_{1}, x_{2}, x_{4}$ & $x_{1}, x_{2}, x_{4}$ & $H$  & $H$ & $\emptyset$ & $H$ \\
\hline
$A_{13}$ & $x_{2}, x_{3}, x_{4}$ & $x_{2}, x_{3}, x_{4}$ & $H$  & $H$ & $\emptyset$ & $H$ \\
\hline
$A_{14}$ & $x_{1}, x_{3}, x_{4}$ & $x_{4}$ & $H$  & $H$ & $\emptyset$ &  $H$\\
\hline
$A_{15}$ & $H$ & $H$ & $H$  & $H$ & $x_1, x_2, x_3$ &  $x_1, x_2, x_3$ \\
\hline
$A_{16}$ & $\emptyset$ & $\emptyset$ & $\emptyset$ & $\emptyset$ & $\emptyset$ & $\emptyset$ \\
\hline
\end{tabular}
\caption{Bited+ GVPRS Approximations}\label{bitedap}
\end{table}

\begin{center}
\begin{table}[hbt]
\begin{tabular}{|c|c|c|c|c|c|}
\hline\hline
\textbf{Set} & $\mathbf{a}$ & $\mathbf{a^{l_1}}$ & $\mathbf{a^{u_1}}$ &  $\mathbf{a^{l_\alpha^*}}$ & $\mathbf{a^{u_\alpha^*}}$ \\
\hline 
$A_{1}$ & $x_{1}$ & $\emptyset$ & $x_1, x_2$ & $x_1, x_2, x_3$ & $x_1, x_2, x_3$\\
\hline
$A_{2}$ & $x_{2}$ & $\emptyset$ & $x_1, x_2, x_3$ & $x_1, x_2, x_3$ & $x_1, x_2, x_3$\\
\hline
$A_{3}$ & $x_{3}$  & $\emptyset$ & $x_1, x_2, x_3$ & $x_1, x_2, x_3$ & $x_1, x_2, x_3$\\
\hline
$A_{4}$ & $x_{4}$ & $\emptyset$ & $\emptyset$ & $\{x_4\}$ & $\{x_4\}$ \\
\hline
$A_{5}$ & $x_{1}, x_{2}$ & $x_1, x_2$ & $x_1, x_2, x_3$ & $x_1, x_2, x_3$  & $x_1, x_2, x_3$\\
\hline
$A_{6}$ & $x_{1}, x_{3}$  & $\emptyset$ & $x_1, x_2, x_3$ & $\emptyset$ & $x_1, x_2, x_3$ \\
\hline
$A_{7}$ & $x_{1}, x_{4}$ & $\emptyset$ & $x_1, x_2, x_3$ & $\emptyset$ & $H$\\
\hline
$A_{8}$ & $x_{2}, x_{3}$  & $x_2, x_3$ & $x_1, x_2, x_3$ & $x_1, x_2, x_3$ & $x_1, x_2, x_3$\\
\hline
$A_{9}$ & $x_{2}, x_{4}$ & $\emptyset$ & $x_1, x_2, x_3$ & $\emptyset$ & $H$\\
\hline
$A_{10}$ & $x_{3}, x_{4}$  & $\emptyset$ & $x_1, x_2, x_3$ &  $x_4$ & $H$\\
\hline
$A_{11}$ & $x_{1}, x_{2}, x_{3}$  & $x_1, x_2, x_3$ & $x_1, x_2, x_3$ &  $x_1, x_2, x_3$ & $x_1, x_2, x_3$ \\
\hline
$A_{12}$ & $x_{1}, x_{2}, x_{4}$ & $x_1, x_2$ & $x_1, x_2, x_3$ & $\emptyset$ & $H$\\
\hline
$A_{13}$ & $x_{2}, x_{3}, x_{4}$ & $ x_2, x_3$ & $x_1, x_2, x_3$ & $\emptyset$ & $H$\\
\hline
$A_{14}$ & $x_{1}, x_{3}, x_{4}$ & $\emptyset$ & $x_1, x_2, x_3$ & $\emptyset$ & $H$\\
\hline
$A_{15}$ & $H$ & $x_1, x_2, x_3$ & $x_1, x_2, x_3$ & $x_1, x_2, x_3$ & $x_1, x_2, x_3$ \\
\hline
$A_{16}$ & $\emptyset$ & $\emptyset$ & $\emptyset$ & $\emptyset$ & $\emptyset$\\
\hline
\end{tabular}
\caption{1-Grade Approximations}\label{gr1}
\end{table}
\end{center}

In the context of this example, relative to the substantial parthood defined by \textbf{bsub2}, that is
$\pc_s ab$ if and only if $(\exists e) e^l\subseteq a^u \subseteq b^u \& \emptyset \subset e^l$, it can be seen that the computation is cumbersome; however, can be simplified by a few if then case-based rules such as
\begin{itemize}
\item {$e$ should be such that $e^l$ is a member of \[\{\{x_4\}, \{x_1, x_2\}, \{x_3, x_2\}, \{x_1, x_2, x_3\}, \{x_1, x_2\, x_4\}, \{x_2, x_3, x_4\}, H \}\]}
\item {If $e^l = H$, then $e= a= H = b$.}
\item {If $e^l= \{x_4\}$, then $e$ can be any of $\{\{x_4\}, \{x_1, x_4\}, \{x_2, x_4\}, \{x_3, x_4\},$\\ $ \{x_1, x_3, x_4\}\}$. Relative to this, the first component of the substantial parthood instance should be such that its upper approximation is $\{x_4\}$ or $H$.   }
\end{itemize}

To see that the relation is distinct from $\subseteq$, it may be noted that \[\pc_s \{x_2, x_3\}\{x_1, x_3, x_4\}\]

\subsubsection{GVPRS Sub-Case}

\begin{proposition}\label{pu}
The parthood $\pc_u^\alpha$, defined by $\pc_{u}^\alpha ab$ if and only if $a^{u_\alpha}\subseteq b^{u_\alpha}$
partitions the powerset $\wp(H)$ into four $H_1=\{A_{16}\}$, $H_2=\{A_1, A_2, A_3, A_5, A_6, A_8, A_{11}, A_{15}\}$, $H_3= \{A_4\}$ and $H_4 =\{A_7, A_9, A_{10},A_{12}, A_{13}, A_{14} \}$ subject to the following rules:
\begin{align}
\text{For each i } (\forall a, b\in H_i) \pc_u^\alpha ab    \tag{1pu}\\
\text{For each i } (\forall a\in H_1) (\forall b\in H_i) \pc_u^\alpha ab  \tag{2pu}\\
(\forall a\in H_2 \cup H_3) (\forall b\in H_4) \pc_u^\alpha ab   \tag{3pu}
\end{align}
\end{proposition}
\begin{proof}
The rules can be discovered from table \ref{bitedap} by grouping similar values in column and comparison. 
 \end{proof}

\begin{proposition}
For $\alpha = 0.3$ and $\kap$ being the usual rough inclusion function, $\pc_s^u$ can be obtained from $\pc_u^\alpha$ of proposition \ref{pu} by removing all pairs of the form $(a, b)$ that fail $\kap (a, b) > 0.3$.  
\end{proposition}
\begin{proof}
This follows because $\pc_s^u ab$ holds if and only if $\pc_u^\alpha ab$ and $\kap(a, b) \geq \alpha$ for any $a$ and $b$. 
\end{proof}

The actual evaluation of $\pc_s^u$ even in this simple example is very lengthy. Furthermore, note that the difference between $\geq$ and $>$ matters.

Suppose that an object to be substantial if it includes the granule $\{x_1, x_2\} = A_5$ or $A_4 = \{x_4\}$. This can be encoded with $\pc_s^t$ by taking $\mathcal{T} = \{A_4, A_5\}$. This defines the rational lower approximation $l_r$ (determined by $l_\alpha$) as a partial function that is defined by the following set of pairs (and undefined for others): \[l_r = \{(A_{4}, A_4), (A_5, A_5), (A_{11}, A_{11}), (A_{15}, A_{11})\}\] 

\subsubsection{Graded Sub-Case}

The computation of the substantial parthoods (except for $\pc_s^1$) in subsection \ref{subpart} are not computationally intensive. $\pc_s^3$ is defined as follows
\[\pc_s^3 {ab} \text{ if and only if } \#(a\cap b) >k \& a\subseteq b  \tag{s3} \]

In the context of this example, for $k= 1$, it can be seen that 
\begin{multline*}
\pc_s^3 = \{(A_5,A_5 ),   (A_5,A_{11} ), (A_5, A_{12}), (A_5, A_{15}), (A_6, A_{6}), (A_6, A_{11}),(A_6, A_{14}),\\
(A_6, A_{15}), (A_7, A_{7}),   (A_7, A_{12}), (A_7, A_{14}), (A_7, A_{15}), (A_8, A_{8}), (A_8, A_{11}),\\ 
(A_8, A_{13}), (A_8, A_{15}), (A_9, A_{9}),   (A_9, A_{12}), (A_9, A_{13}), (A_9, A_{15}), (A_{10}, A_{10}),\\
(A_{10}, A_{13}),(A_{10}, A_{14}), (A_{10}, A_{15}), (A_{11}, A_{11}), (A_{11}, A_{15}), (A_{12}, A_{12}),  (A_{12}, A_{15}),\\ 
(A_{13}, A_{13}), (A_{13}, A_{15}),(A_{14}, A_{14}), (A_{14}, A_{15}),  (A_{15},A_{15} )\}  
\end{multline*}

In the \textsf{pGpsGS} framework, the context of this example can be straightforwardly interpreted  with the above substantial parthood. 

\begin{itemize}
\item {Now consider the set $\{x_1, x_2, x_4\}$. From the table, it is clear that \[\{x_1, x_2, x_4\}^{l_1} = \{x_1, x_2\} = A_5\]}
\item {It can be checked that this is an instance of a substantial lower approximation from the definition (relative to $\pc_s^3$).}
\item {All other instances of $l_1$ can be similarly checked. }
\item {Thus $l_1$ is an example of a rational lower approximation (this is proved in a forthcoming paper by the first author).}
\item {By contrast $l$ is not a rational lower approximation relative to $\pc_s^3$.}
\end{itemize}

\begin{remark}
Furthermore, this example shows that definitions of substantial parthoods that involve approximations may have additional computational load to deal with.    
\end{remark}

\section{Granular VPRS and the Framework}

Questions about the properties satisfied by intuitively sensible ideas of parthood stated in section \ref{gvprssu} are evaluated in relation to the definitions in the framework. $\kap$ will be assumed to be a wqRIF unless specified otherwise (that is, it satisfies \textsf{R0} and \textsf{R3}).

\begin{align*}
\pc_s^{5*} {ab} \text{ if and only if } a^{l_{\alpha^*}}\subseteq b^{l_{\alpha^*}}   \tag{s5*}\\
\pc_s^6 {ab} \text{ if and only if } a\subseteq b \& \#(a)>k   \tag{s6}\\
\pc_s^7 {ab} \text{ iff } (\forall h\in \mathcal{G})(h\subseteq a,\, \&\, h\subseteq b, \&\,\kap (a, h) \geq 1- \alpha \longrightarrow \kap (b, h) \geq 1- \alpha )   \tag{s7}\\
\pc_s^* {ab} \text{ if and only if } \#(a\cap b) >k \& b\not\subset a  \tag{s*}\\
\pc_s^9 {ab} \text{ if and only if } (\forall h\in \mathcal{G})(\kap (a, h) \geq \alpha \longrightarrow \kap (b, h) \geq \alpha )   \tag{s9}\\
\pc_s^3 {ab} \text{ if and only if } \#(a\cap b) >k \& a\subseteq b  \tag{s3}\\
\pc_s^l ab \text{ if and only if } a^{l_\alpha}\subseteq b^{l_\alpha} \,\&\, \kap(a, b) \geq 1-\alpha  \tag{s0l}\\
\pc_s^u ab \text{ if and only if } a^{u_\alpha}\subseteq b^{u_\alpha} \,\&\, \kap(a, b) \geq \alpha  \tag{s0u}
\end{align*}

\begin{align}
(\forall a, b)\, (\pc_s ab \& \pc_s ba \longrightarrow a\approx b)   \tag{sub3}\\
(\forall a, b, e)\, (\pc_s ae \& \pc_s ab \longrightarrow  \pc_s a(b\vee e) )    \tag{sub4}
\end{align}
If in addition, \textsf{sub1, sub2, sub5}, \textsf{UL1} and \textsf{sub6} are satisfied, then $\mathbb{S}$ will be said to be a \emph{Pre-General Substantial Granular Space} (\textsf{pGsGS}).
\begin{align}
(\forall a)\, \pc_s aa    \tag{sub1}\\
(\forall a, b)\, (\pc_s ab \longrightarrow \pc ab)     \tag{sub2}\\
(\forall a, b, e)\, (\pc_s ba \& \pc_s be \& \pc ae \longrightarrow  \pc_s ae)   \tag{sub5}\\
(\forall a, b, e)\, (\pc_s ab \& \pc_s eb \& \pc ae \longrightarrow  \pc_s ae)   \tag{sub6}
\end{align}

\begin{proposition}
$\pc_s^3$ satisfies \textsf{sub2, sub3, sub4, sub5}, and \textsf{sub6}. Further it satisfies 
\[(\forall a, b)\, (\pc_s ab \& \pc_s ba \longrightarrow a = b)   \tag{Asy}\]
\end{proposition}

\begin{proof}
Reflexivity fails in general because not every object $a$ needs to have $k$ elements for a finite $k$.
\begin{enumerate}
\item {Since $\subseteq$ is antisymmetric, therefore $\pc_s^3$ satisfies \textsf{Asy} and \textsf{sub3}.}
\item {If for any $a, b$ and $e$, $\pc_s^2 ae \& \pc_s^2 ab$, then $\#(a\cap e) >k$, $\#(a\cap b) >k$, $a\subseteq e$, and $a\subseteq b$. Thus $a\subseteq (e\cup b) $. This proves \textsf{sub4}.}
\item {\textsf{sub2} holds because $\subseteq$ is part of the defining condition of $\pc_s^3$.}
\item {\textsf{sub5} holds because if for any $a, b$ and $e$, $\pc_s^2 ba \& \pc_s^2 be$ and $a\subseteq e$, then $\#(a\cap b) >k$, $\#(e\cap b) >k$ ensures that $\#(a\cap e) >k$. }
\item {\textsf{sub6} holds because if for any $a, b$ and $e$, $\pc_s^2 ab \& \pc_s^2 eb$ and $a\subseteq e$, then $\#(a\cap b) >k$, $\#(e\cap b) >k$ ensures that $\#(a\cap e) >k$.}
\end{enumerate} 
\end{proof}

\begin{proposition}
$\pc_s^5$ satisfies \textsf{sub3, sub1}, and \textsf{sub2}. However does not satisfy \textsf{sub4, sub5} and \textsf{sub6} in general.
\end{proposition}
\begin{proof}
The predicate is defined by $\pc_s^5 {ab} \text{ if and only if } a^{l_\alpha}\subseteq b^{l_\alpha} $.

Thus if $a=b$, then $\pc_s^5 aa$ holds and \textsf{sub1} holds.
\end{proof}

\begin{proposition}
$\pc_s^6$ satisfies \textsf{sub1, sub2, sub3, sub4}, and \textsf{sub6}. However it does not satisfy \textsf{sub5} in general.
\end{proposition}

\begin{proposition}
$\pc_s^7$ satisfies \textsf{sub3}, and \textsf{sub1}. However does not satisfy \textsf{sub2, sub4, sub5} and \textsf{sub6} in general.
\end{proposition}

\begin{proposition}
$\pc_s^9$ satisfies \textsf{sub3}, and \textsf{sub1}. However does not satisfy \textsf{sub2, sub4, sub5} and \textsf{sub6} in general.
\end{proposition}
\begin{proof}
The predicate is defined by $\pc_s^9 {ab} \text{ if and only if } (\forall h\in \mathcal{G})(\kap (a, h) \geq \alpha \longrightarrow \kap (b, h) \geq \alpha )$

For $a=b$, $\pc_s^9 aa$ obviously holds. Thus \textsf{sub1} is satisfied.

The property \textsf{sub3} holds when $\approx$ is interpreted as $(\forall h\in \mathcal{G})(\kap (a, h) \geq \alpha \leftrightarrow \kap (b, h) \geq \alpha )$. 

The other properties do not necessarily hold.
\end{proof}

\begin{proposition}
$\pc_s^l$ satisfies \textsf{sub3}, and \textsf{sub1}. However does not satisfy \textsf{sub2, sub4, sub5} and \textsf{sub6} in general.
\end{proposition}
\begin{proof}
The predicate is defined by $\pc_s^l ab \text{ if and only if } a^{l_\alpha}\subseteq b^{l_\alpha} \,\&\, \kap(a, b) \geq 1-\alpha$ 

For any $a$ and $b$, $a \approx b$ needs to be interpreted as $a^{l_\alpha} = b^{l_\alpha} \,\&\, \kap(a, b) \geq 1-\alpha \, \&\, \kap (b, a)\geq 1 -\alpha$ to guarantee \textsf{sub3}
\end{proof}

\begin{proposition}
$\pc_s^u$ satisfies \textsf{sub3}, and \textsf{sub1}. However does not satisfy \textsf{sub2, sub4, sub5} and \textsf{sub6} in general.
\end{proposition}

\section{Applications}

Novel meta-applications are outlined in this section as VPRS, and hybrid variants are too frequently used in the literature. Further adaptations of the intelligent grain sorting meta-algorithm proposed in \cite{am202236} using graded rough sets is modified for the present context. This is potentially applicable to situations where sensors are not capable enough to handle the load (or restricted by design limitations).

\subsection{New Cluster Validation Technique}

In \cite{am2021c}, general rough set based methods are proposed for validation of soft and hard clustering contexts. The methods consist in modeling the entire clustering context from a granular rough set based perspective without intrusion and contamination, and measuring the deviation from rational approximations. 

The use of VPRS and related approximations cannot be expected to be non-intrusive or contamination free to a reasonable extent because of the fundamental role played by generalized RIFS. However can still provide reasonable comparisons of validity for different types of attributes (especially ones that are additive) and only when concepts of substantial parthood and rational approximations are definable.

The proposed methodology automatically makes it semi-supervised at the interpretation stage, and is essentially validation against an approximate reasoning perspective. The strategy is outlined in figure \ref{clumo} and the steps are described below.

\begin{figure}[hbt]
\centering
 \begin{tikzpicture}[node  distance=2.0cm, auto]
\node (A0) {Clean Data};
\node (F0) [left of=A0]{};
\node (A1) [below of=A0] {};
\node (B1) [left of=F0] {Ontology};
\node (F1) [right of=A1]{};
\node (A2) [right of=F1] {Soft/Hard-Clustering};
\node (B2) [below of=A1] {Relations-of-Interest};
\node (B3) [below of=B2] {GVPRS Model};
\node (C1) [below of=B3] {Substantial Parthoods};
\node (C2) [below of=C1] {Rational Classes};
\node (A3) [below of=C2] {Comparison};
\draw[->] (A0) to node {}(A2);
\draw[->] (A0) to node {}(B1);
\draw[->] (A0) to node {}(B2);
\draw[->] (A2) to node {}(A3);
\draw[->] (B1) to node {}(B2);
\draw[->] (B2) to node {}(B3);
\draw[->] (B3) to node {}(C1);
\draw[->] (C1) to node {}(C2);
\draw[->] (C2) to node {}(A3);
\draw[dashed] (B1) to node {}(A3);
\end{tikzpicture}
\caption{Cluster Validation by Models}\label{clumo}
\end{figure}
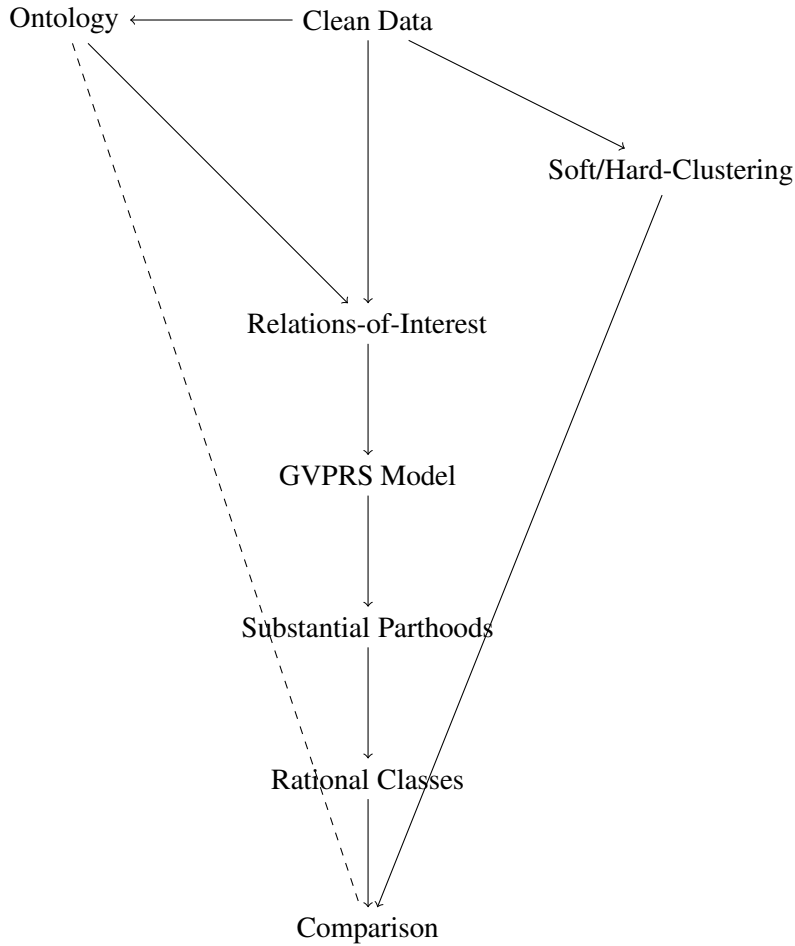

\begin{description}
 \item [Semantics]{Model the situation in the granular perspective using a suitable variant of the framework -- assuming that an interpretation using generalized granular VPRS is sensible in the situation (for the context of this research paper).}
 \item [Data]{Start with clean data.}
 \item [Ontology]{Specify relevant ontology in the context on the basis of meaning (in formal or informal terms). }
 \item [Clustering]{Use any soft or hard clustering technique in the light of available ontology.}
 \item [Relations]{Identify relations of interest on the basis of ontological considerations.}
 \item [GVPRS Model]{Compute the generalized VPRS model for fixed sets of values of parameters}
 \item [Rational Classes]{Compute the approximations and rational approximations of the clusters.}
 \item [Comparison]{Verify the closeness of clusters to rational approximations using the measures of \cite{am2021c}. Variants of the measures are investigated in a forthcoming paper by the first author.}
\end{description}

Modification of the parameters or even the generalized rough inclusion function may be required to improve the fit and associate an ontology. Moreover, the methodology is not intended to be unsupervised because at least some understanding of the data and context is necessary to specify the appropriate generalized RIFs, the parameters, choice of clustering techniques, and measure classes. 

\subsection{Problems in Medical Imaging}

In medical imaging, problems of identifying regions and volumes of interest that contain specific features characterizing medical conditions, malignancies or tumors is difficult for both radiologists and machine learning scientists. Previous research has helped in the formulation of methods of machine learning aided radiology based diagnosis, classification, treatment planning, prognosis, and post-treatment evaluation.  Medical characterizations have their grey regions due to gaps and shortcomings in medical research. These may be associated with a number of possible solutions/mechanisms, and it would be useful to consider automated discovery of patterns and rules across imaging data. Applications of VPRS and hybrid versions thereof to image processing and related feature selection has a long history, aspects of which are mentioned in \cite{ljy2020,angelma2016,ajg2008}. 

Image segmentation is used to assign labels to the pixels of an image (obtained through methods such as magnetic resonance imaging (MRI), computed tomography (CT), and positron emission tomography (PET)) based on shared visual attributes -- this helps in approximately identifying regions of interest. This applies to automatic detection and diagnostics. Most research on ML methods in the context focus on these problems (see \cite{ajg2008}); however, not on discovering other patterns. In fact, noise-removal techniques like multi-level thresholding that are always used on the image data can be suspected to contribute to significant data loss. 

In medical image segmentation problems, the idea of misclassification is often treated in a simplistic way. If a tumor appears to be next to an apparently good region, then the latter region is most likely to have been affected in a complex way. One way to identify different regions can be through generalized rough inclusion functions. In the literature on cancer research, surgical removal of tumors is known to trigger metastases \cite{tost2017}. This leads to the problem of finding patterns and possibly useful rules that characterize the process of their origin, and states that are likely to degenerate. If it is assumed as in \cite{bms2016}, that the main problem is of image segmentation for precisely identifying malignant areas or problem zones, then this can be masking the possibilities These may be spread over a few areas of the image or may be localized in a single area. The malignancy in different areas may additionally be characterized by distinct feature vectors. This motivates studies relating to decision-making and rule discovery. For example, the extended VPRS methods that are intended for probabilistic set-valued information systems (or rather, probabilistic indeterministic information systems) of \cite{ytchs2017} are possibly modifiable from the perspective of this research for application to rule discovery in the contexts of this subsection.

The detailed meta-analysis in \cite{lrdm2020} demonstrates that fMRI methods are not reliable because the empirical evidence for brains working consistently along fixed patterns (in the perspective of fMRI) is not in place. This invalidates a number of related studies and methods. While rough set have been applied to some of these, especially to noise reduction and feature selection, this affects the prospect of improving these. On the other hand, the recent development of portable MRI scanners \cite{llz2021} that use low intensity magnetic fields and produce images with more noise, pose additional computational modeling challenges.

\subsection{Dynamic Approximation of Mixed Quality Fluids}

This example does not involve explicit definitions of relations on the dynamic data, and is analogous to the method of dynamic food grain sorting and quality approximation through graded rough sets in \cite{am202236}. In the example, ingredients of different types and qualities are sorted, and dynamically combined to form blends that approximate acceptable target qualities. While grains can be classified with the help of sensors and graded rough sets, not every ingredient can be processed the same way. Furthermore, it might be useful to treat flowing stream of discrete entities as a fluid as is often done in physics to accommodate non-discrete measures such as volume or mass. 

Dynamic intelligent quantity-aware methods of ingredient classification are necessary for quality assurance in the food processing or pharmaceutical industry. If raw materials such as milk, milk products, seeds, grains, and oil, are procured from diverse sources, then variations in quality are common. Manufacture of pharmaceutical raw materials involves additional constraints on  environment, procedures, and quality.   

A survey of computer vision based methods can be found in \cite{hpra2021}. While the intent of the methods are for dynamic classification, the methods do not follow universal standards, often reinvent the wheel, are proprietary and have other limitations. Oversimplified color based sorting methods moreover lead to loss of efficiency. Whenever the computational power required to collect more data from \emph{flowing input} is not too high, and it is possible to use rational variable precision or graded rough sets for classification and approximation.

Suppose that the attribute value sets of the input product are  $\{V_{a}:\, a\in \{1, 2, \ldots n\}\}$. An $f$-valued classification (with $f\geq 3$) into the disjoint categories: \textsf{reject}, \textsf{perfect}, \textsf{sub-perfect-$(1)$} \dots \ \textsf{sub-perfect-$(f-2)$} can be assumed to be the result of the procedure. Similar sorting patterns can be assumed for other types of ingredients. A subsequent transfer of the things falling under the category \textsf{reject} to the \textsf{reject} bin, \textsf{perfect} to the \textsf{perfect} bin and \textsf{sub-perfect-$i$} to the \textsf{sub-perfect-$i$} bin ensures an initial classification. This stage can moreover involve rational VPRS or other soft methods (the classification by themselves require concepts of substantial parthood). However it is the use in subsequent stages that will be highlighted in this example.

\begin{figure}[hbt]
\begin{center}
\begin{tikzpicture}[node distance=2cm, auto]
\node (F) {Produce};
\node (O) [left of=F] {};
\node (E) [left of=O] {Categories};
\node (C) [below of=O] {Bins};
\node (A) [right of=F] {R-Bins};
\node (G) [below of=C] {};
\node (B) [right of=G] {Granulations};
\node (K) [below of=B] {G-Approximations};
\draw[->,font=\scriptsize] (F) to node {}(C);
\draw[->,font=\scriptsize] (E) to node {}(C);
\draw[->,font=\scriptsize] (A) to node {}(C);
\draw[->,font=\scriptsize] (C) to node {}(B);
\draw[->,font=\scriptsize] (E) to node {}(K);
\draw[->,font=\scriptsize] (B) to node {}(K);
\end{tikzpicture}
\caption{Industrial Application}
\label{cake3}
\end{center}
\end{figure}
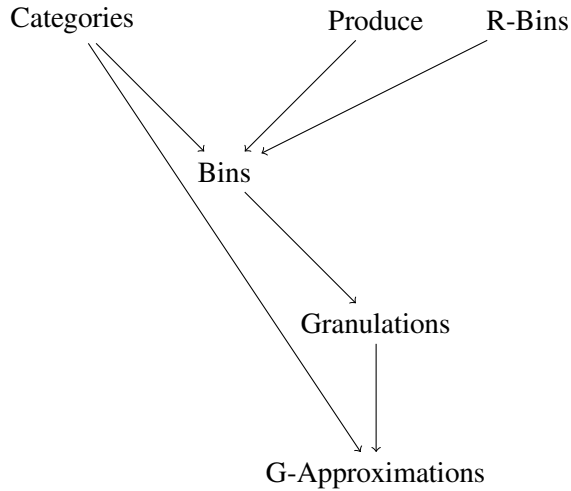

After a time $t$ of sorting the produce into $f$ number of bins, the problem would be of combining them to form blocks that satisfy the quality criteria. It is assumed that some bins have residual material from previous runs. The problem may moreover be read as a number of optimization problems (as it is not required to exhaust bins in each cycle), that may or may not have solutions; however, such an approach would require human intervention and decision-making.

\begin{itemize}
\item {A granule $h$ is a standardized set of objects determined by a set $V(h)$ of $n$-tuples of values that are computed to be admissible (from the mass of ingredients and associated classification) according to an external criteria. For example, an ingredient of type-1 and another ingredient of type-2 in the ratio $1:2$ may be known to satisfy the criteria. However there is no need to explicitly mention the ratio}
\item {Granular rational lower VPRS approximations are combinations from the bins that satisfy the criteria}
\item {Granular rational upper VPRS approximations on the other hand can be used to satisfy a lower quality criteria.}
\item {The objective is to produce admissible combinations of ingredients to approximate the quality criteria.}
\item {The associated mereological features are explained in the next sub-section.}
\end{itemize}

\subsubsection{Approximation of Desired Quality}

Suppose the granules are collections of ingredients of the form $x_ij$ with $i$ indicating the ingredient and $j$ it's type. $ij$ can be assumed to be the bin identity indicator as well. It is assumed that bin $ij$ contains many approximate copies of $x_{ij}$.  Note that if the sorting is imperfect, then this would mean that \emph{most of the content of bin $ij$ is of the $j$th type of ingredient $i$}. For this example, in addition, assume that there are three ingredients each of five types. A sum formula like $x_{11} \oplus x_{21} \oplus x_{33}$ is intended to mean a weighted combination of the granules in some ratio. In the simplest case this can be a union of a number of granules.

Now, a rule like there should be \emph{at least two units of ingredients of perfect type, and at least three units of ingredients for an acceptable combination} is equivalent to constructing a lower VPRS approximation. For a purely set theoretic version the granules can be numbered in addition (relative to the mass of ingredients).

While a combination need not exactly correspond to a quality grade, the goal is to approximate it. Thus a granular lower VPRS approximation can assure a certain quality, while a granular upper VPRS approximation can possibly assure a certain quality. These are also rational in suitable perspectives implicit in their construction.

\section{Remarks and Directions}

In this research, 
\begin{itemize}
 \item {granular generalizations of VPRS that are different from the generalizations in the DTRS and probabilistic perspectives are introduced,}
 \item {new results on connections with graded granular rough sets are proved,}
 \item {generalizations of inclusion functions are systematized and expanded,}
 \item {concepts of substantial parthood and rational approximations are introduced,}
 \item {a framework for rational granular graded approximations due to the first author in \cite{am202236} is improved and extended,}
 \item {variations of the framework for granular VPRS models are explored,}  
 \item {it is shown that particular combinations of parthood and approximations work better,}
 \item {meta-applications to rough-set based cluster analysis introduced in \cite{am2021c} are proposed,}
 \item {medical imaging and tumor metastases prediction are outlined, and }
 \item {applications to sensor-based dynamic approximate quality assurance through dynamic sorting are proposed.}
\end{itemize}

This research situates ideas of rationality in relation to granularity, and substantial part of relations in the generalized VPRS context motivates work from both theoretical and applied perspectives. Important connections with defeasible reasoning \cite{dm2003} are derived -- these suggest that the most rational kinds of VPRS approximations must be governed by a set of higher order ontological rules (associated with the definition of approximations). In future work, associated theoretical and application areas will be investigated in greater depth.

\begin{small}
\begin{flushleft}
\textbf{Acknowledgement:} The first author's contribution to this research is supported by a woman scientist grant (grant no. WOS-A/PM-22/2019) of the Department of Science and Technology.
\end{flushleft}
\end{small}

\bibliographystyle{fundam}
\bibliography{algroughf69flzf}

\end{document}